\documentclass[a4paper,10pt]{article}
\usepackage[utf8]{inputenc}

\usepackage{cite}

\usepackage{graphicx}
\usepackage{booktabs}

\usepackage[cmex10]{amsmath}
\usepackage{amsfonts}
\usepackage{amsthm}
\usepackage{enumerate}

\usepackage{rotating}
\usepackage{algorithmic}
\usepackage{algorithm}

\usepackage{array}
\usepackage{dblfloatfix}
\usepackage{url}

\newtheorem{theorem}{Theorem}
\newtheorem{remark}{Remark}
\usepackage{ bbold }

\usepackage{color}
\usepackage[table]{xcolor}

\newcommand\R{\mathbb{R}}

\newcommand\nor{\mathcal{N}}
\newcommand\cov{\mathrm{cov}}
\newcommand\var{\mathrm{var}}
\newcommand\std{\mathrm{std}}

\newcommand\covlw{\mathrm{cov_{\dagger}}}

\newtheorem{problem}{Problem}
\usepackage{pgf}
\usepackage{tikz}

\usepackage{MnSymbol}
\newcommand\de[1]{\lsem #1 \rsem}

\newcommand\Hspace{\mathcal{H}}
\newcommand\Hlayer{\mathbf{H}}
\newcommand\hpoint{\textsc{h}}
\newcommand\Yset{\mathbf{T}}
\newcommand\ypoint{\mathbf{t}}

\newcommand\Xspace{\mathcal{X}}
\newcommand\Xlayer{\mathbf{X}}
\newcommand\xpoint{\textsc{x}}

\newcommand\Gfunction{\mathbf{G}}
\newcommand\Kernel{\mathbf{K}}
\newcommand\Ffunction{\mathbf{F}}

\newcommand\subsample{^{[h]}}

\newcommand\Prob{p}
\newcommand\LinearOperator{\boldsymbol{\beta}}

\newcommand\Dcs{{D}_\textsc{cs}}
\newcommand\Dkl{{D}_\textsc{kl}}

\newcommand\meanpoint{m}

\DeclareMathOperator*{\argmax}{arg\,max}

\title{Extreme Entropy Machines:\\ Robust information theoretic classification}
\date{}
\author{Wojciech Marian Czarnecki, Jacek Tabor
\\ Faculty
of Mathematics and Computer Scicence, \\ Jagiellonian University, Krakow, Poland,\\
 e-mail: \{wojciech.czarnecki, jacek.tabor\}@uj.edu.pl 
}

\begin{document}

\maketitle

\begin{abstract}
Most of the existing classification methods are aimed at minimization of empirical risk (through some simple point-based error measured with loss function) with added regularization. We propose to approach this problem in a more information theoretic way by investigating applicability of entropy measures as a classification
model objective function. We focus on quadratic Renyi's entropy and connected Cauchy-Schwarz Divergence which leads to the construction of Extreme Entropy Machines (EEM).

The main contribution of this paper is proposing a model based on the information theoretic concepts which on the one hand shows new, entropic perspective on known linear classifiers and on the other leads to a construction of very robust method competetitive with the state of the art non-information theoretic ones (including Support Vector Machines and Extreme Learning Machines). 

Evaluation on numerous problems spanning from small, simple ones from \textsc{UCI} repository to the large (hundreads of thousands of samples) extremely unbalanced (up to 100:1 classes' ratios) datasets shows wide applicability of the EEM in real life problems and that it scales well.

\end{abstract}

\section{Introduction}
%
%
%
%


There is no one, universal, perfect optimization criterion that can be used to train machine learning model. Even for linear classifiers one can find multiple objective functions,
error measures to minimize, regularization methods to include~\cite{kulkarni1998learning}. Most of the existing methods are aimed at minimization of empirical risk (through some simple point-based error measured
with loss function) with added regularization. We propose to approach this problem in more information theoretic way by investigating applicability of entropy measures as a classification
model objective function. We focus on quadratic Renyi's entropy and connected Cauchy-Schwarz Divergence.

One of the information theoretic concepts which has been found very effective in machine learning is the entropy measure. In particular the rule of maximum entropy modeling led to the construction of MaxEnt model and its structural generalization -- Conditional Random Fields which are considered state of the art in many applications. In this paper we propose to use Renyi's quadratic cross entropy as the measure of two density estimations divergence in order to find best linear classifier. It is a conceptually different approach than typical entropy models as it works in the input space instead of decisions distribution. As a result we obtain a model closely related to the Fischer's Discriminant (or more generally Linear Discriminant Analysis) which deepens the understanding of this classical approach. Together with a powerful extreme data transformation we obtain a robust, nonlinear model competetive with the state of the art models not based on information theory like Support Vector Machines (SVM~\cite{
Vapnik95}), Extreme Learning Machines (ELM~\cite{huang2004extreme}) or Least Squares Support Vector Machines (LS-SVM~\cite{suykens1999least}). We also show that under some simplifing assumptions ELM and LS-SVM can be seen through a perspective of information theory as their solutions are (up to some constants) identical to the ones obtained by proposed method.


Paper is structured as follows: first we recall some preliminary information regarding ELMs and Support Vector Machines, including Least Squares Support Vector Machines. Next we introduce our Extreme Entropy 
Machine (EEM) together with its kernelized extreme counterpart -- Extreme Entropy Kernel Machine (EEKM). We show some connections with existing models and some different perspectives
for looking at proposed model. In particular, we show how learning capabilities of EEMs (and EEKM) reasamble those of ELM (and LS-SVM respectively). During evaluation on over 20 binary
datasets (of various sizes and characteristics) we analyze generalization capabilities and learning speed. We show that it can be a valuable, robust alternative for existing methods. 
In particular, we show that it achieves analogous of ELM stability in terms of the hidden layer size. We conclude with future development plans and open problems.

\section{Preliminaries}

Let us begin with recalling some basic information regarding Extreme Learning Machines~\cite{huang2006extreme} and Support Vector Machines~\cite{Vapnik95} which are further used as a competiting models for proposed solution.
We focus here on the optimization problems being solved to underline some basic differences between these methods and EEMs.

\subsection{Extreme Learning Machines}

Extreme Learning Machines are relatively young models introduced by Huang et al.~\cite{huang2004extreme} which are based on the idea that single layer feed forward neural networks (SLFN) can be trained without iterative process by performing linear regression on the data mapped through random, nonlinear projection (random hidden neurons). More precisely speaking, basic ELM architecture consists of $d$ input neurons connected with each input space dimension which are fully connected with $h$ hidden neurons by the set of weights $w_j$ (selected randomly from some arbitrary distribution) and set of biases $b_j$ (also randomly selected). Given some generalized nonlinear activation function $\Gfunction$ one can express the hidden neurons activation matrix $\Hlayer$ for the whole training set $\Xlayer, \Yset = \{(\xpoint_i,\ypoint_i)\}_{i=1}^N$ such that $\xpoint_i \in \R^d$ and $\ypoint_i \in \{-1.+1\}$ as
$$
\Hlayer_{ij} = \Gfunction(\xpoint_i,w_j,b_j).
$$
we can formulate following optimization problem

\medskip

\noindent{\textbf{Optimization problem: Extreme Learning Machine} }
\begin{equation*}
\begin{aligned}
& \underset{\LinearOperator}{\text{minimize}}
& & \left \| \Hlayer\LinearOperator - \Yset \right \|^2  \\
& \text{where}
& & \Hlayer_{ij} = \Gfunction(\xpoint_i,w_j,b_j). ,\;  i = 1, \ldots ,N, j = 1, \ldots , h\\
\end{aligned}
\end{equation*}

\medskip
If we denote the weights between hidden layer and output neurons by $\LinearOperator$ it is easy to show~\cite{huang2006extreme} that putting
$$
\LinearOperator = \Hlayer^\dagger \Yset,
$$
gives the best solution in terms of mean squared error of the regression:
$$
\left \| \Hlayer\LinearOperator - \Yset \right \|^2 =\left \| \Hlayer(\Hlayer^\dagger \Yset) - \Yset \right \|^2 = \min_{\alpha \in \mathbb{R}^{h}} \left \| \Hlayer\alpha - \Yset \right \|^2
$$
where  $\Hlayer^\dagger$ denotes Moore-Penrose pseudoinverse of matrix $\Hlayer$. 

Final classification of the new point $x$ can be now performed analogously by classifying according to
$$ cl(\xpoint) = \textsc{sign}( \left [\begin{matrix}
                           \Gfunction(\xpoint,w_1,b_1) & \ldots & \Gfunction(\xpoint,w_d,b_d) 
                           \end{matrix}
\right ]\LinearOperator ). $$
%
As it is based on the oridinary least squares optimization, it is possible to balance it in terms of unbalanced datasets by performing weighted ordinary least squares. In such a scenario, given a vector $B$ such that $B_i$ is a square root of the inverse of the $\xpoint_i$'s class size and $B \cdot X$ denotes element wise multiplication between $B$ and $X$:

$$
\LinearOperator = (B \cdot \Hlayer)^\dagger B \cdot \Yset
$$

\subsection{Support Vector Machines and Least Squares Support Vector Machines}

One of the most well known classifiers of the last decade is Vapnik's Support Vector Machine (SVM~\cite{Vapnik95}), based on the principle of creating linear classifier that maximizes the separating margin between elements of two classes.

\medskip

\noindent{\textbf{Optimization problem: Support Vector Machine} }
\begin{equation*}
\begin{aligned}
& \underset{\LinearOperator,b.\xi}{\text{minimize}}
& & \frac{1}{2} \left \| \LinearOperator \right \|^2 + C \sum_{i=1}^N \xi_i \\
& \text{subject to}
& & \ypoint_i(\langle \LinearOperator, \xpoint_i\rangle + b) = 1 - \xi_i  ,\;  i = 1, \ldots ,N
\end{aligned}
\end{equation*}

\medskip
\noindent which can be further kernelized (delinearized) using any kernel $\Kernel$ (valid in the Mercer's sense):
\medskip

\noindent{\textbf{Optimization problem: Kernel Support Vector Machine} }
\begin{equation*}
\begin{aligned}
& \underset{\LinearOperator}{\text{maximize}}
& & \sum_{i=1}^N\LinearOperator_i - \frac{1}{2} \sum_{i,j=1}^N \LinearOperator_i \LinearOperator_j \ypoint_i \ypoint_j \Kernel(\xpoint_i,\xpoint_j) \\
& \text{subject to}
& & \sum_{i=1}^N \LinearOperator_i \ypoint_i = 0 \\
& & & 0 \leq \LinearOperator_i \leq C   ,\;  i = 1, \ldots ,N
\end{aligned}
\end{equation*}

\medskip

The problem is a quadratic optimization with linear constraints, which can be efficiently solved using quadratic programming techniques. Due to the use of hinge loss function on $\xi_i$ SVM attains very sparse solutions in terms of nonzero $\LinearOperator_i$. As a result, classifier does not have to remember the whole training set, but instead, the set of so called support vectors ($SV = \{ \xpoint_i : \LinearOperator_i > 0 \}$), and classify new point according to

$$
 cl(\xpoint) = \textsc{sign} \left (  \sum_{\xpoint_i \in SV} \LinearOperator_i \Kernel(\xpoint_i, \xpoint) + b \right ).
$$

It appears that if we change the loss function to the quadratic one we can greatly reduce the complexity of the resulting optimization problem, leading to the so called Least Squares Support Vector Machines (LS-SVM).

\medskip

\noindent{\textbf{Optimization problem: Least Squares Support Vector Machine} }
\begin{equation*}
\begin{aligned}
& \underset{\LinearOperator,b.\xi}{\text{minimize}}
& & \frac{1}{2} \left \| \LinearOperator \right \|^2 + C \sum_{i=1}^N \xi^2_i \\
& \text{subject to}
& & \ypoint_i(\langle \LinearOperator, \xpoint_i\rangle + b) = 1 - \xi_i  ,\;  i = 1, \ldots ,N
\end{aligned}
\end{equation*}

\medskip

and decision is made according to
$$
 cl(\xpoint) = \textsc{sign}( \langle \LinearOperator, \xpoint \rangle + b )
$$
As Suykens et al. showed~\cite{suykens1999least} this can be further generalized for abitrary kernel induced spaces, where we classify according to:

$$
 cl(\xpoint) = \textsc{sign} \left (  \sum_{i=1}^N \LinearOperator_i \Kernel(\xpoint_i, \xpoint) + b \right )
$$
where $\LinearOperator_i$ are Lagrange multipliers associated with particular training examples $\xpoint_i$ and $b$ is a threshold, found by solving the linear system

$$
\left [
\begin{matrix}
 0 & \mathbb{1}^T \\
 \mathbb{1} & \Kernel(\Xlayer,\Xlayer) + I/C
\end{matrix}
\right ]
\left [
\begin{matrix}
b\\
\LinearOperator
\end{matrix}
\right ]
=
\left [
\begin{matrix}
 0 \\
 \Yset
\end{matrix}
\right ]
$$
where $\mathbb{1}$ is a vector of ones and $I$ is an identity matrix of appropriate dimensions.
Thus a training procedure becomes

$$
\left [
\begin{matrix}
b\\
\LinearOperator
\end{matrix}
\right ]
=
\left [
\begin{matrix}
 0 & \mathbb{1}^T \\
 \mathbb{1} & \Kernel(\Xlayer,\Xlayer) + I/C
\end{matrix}
\right ]^{-1}
\left [
\begin{matrix}
 0 \\
 \Yset
\end{matrix}
\right ].
$$
Similarly to the classical SVM, this formulation is highly unbalanced (it's results are skewed towards bigger classes). To overcome this issue one can introduce a weighted version~\cite{suykens2002weighted}, where given diagonal matrix of weights $Q$, such that $Q_{ii}$ is invertibly proportional to the size of $\xpoint_i$'s class and .

$$
\left [
\begin{matrix}
b\\
\LinearOperator
\end{matrix}
\right ]
=
\left [
\begin{matrix}
 0 & \mathbb{1}^T \\
 \mathbb{1} & \Kernel(\Xlayer,\Xlayer) + Q /C
\end{matrix}
\right ]^{-1}
\left [
\begin{matrix}
 0 \\
 \Yset
\end{matrix}
\right ].
$$

Unfortunately, due to the introduction of the square loss, the Support Vector Machines sparseness of the solution is completely lost. Resulting training has a closed form solution, but requires the computation of the whole Gram matrix and the resulting machine has to remember\footnote{there are some pruning techniques for LS-SVM but we are not investigating them here} whole training set in order to perform new point's classification.

\section{Extreme Entropy Machines}

Let us first recall the formulation of the linear classification problem in the highly dimensional feature spaces, ie. when number of samples $N$ is equal (or less) than 
dimension of the feature space $h$. In particular we formulate the problem in the limiting case\footnote{which is often obtained
by the kernel approach} when $h=\infty$:

\begin{problem}
We are given finite number of (often linearly independent)
points $\hpoint_i^\pm$ in an infinite dimensional Hilbert space $\Hspace $. Points $\hpoint^+ \in \Hlayer^+$ constitute the
positive class, while $\hpoint^- \in \Hlayer^-$ the negative class. 

We search for $\LinearOperator \in \Hspace $ such that 
the sets $\LinearOperator^T\Hlayer^+$ and $\LinearOperator^T\Hlayer^-$ are {\em optimally separated}.
\end{problem}

Observe that in itself (without additional regularization) the problem is not well-posed
as, by applying the linear independence of the data, for 
arbitrary $\meanpoint_+ \neq \meanpoint_-$ in $\R$ we can
easily construct $\LinearOperator \in \Hspace $ such that 
\begin{equation} \label{over}
\LinearOperator^T\Hlayer^+=\{\meanpoint_+\} \text{ and } \LinearOperator^T\Hlayer^-=\{\meanpoint_-\}.
\end{equation}
However, this leads to a model case of overfitting, which typically yields suboptimal results on the testing set (different from the orginal training samples).

To make the problem well-posed, we typically need to:
\begin{enumerate}
\item add/allow some error in the data,
\item specify some objective function including term penalising model's complexity.
\end{enumerate}

%

Popular choices of the objective function include per-point classification loss (like square loss in neural networks or hinge loss in SVM) with a regularization term added, often expressed as the square of the norm of our operator $\LinearOperator$ (like in SVM or in weight decay regularization of neural networks). In general one can divide objective functions derivations into following categories:
\begin{itemize}
 \item regression based (like in neural networks or ELM),
 \item probabilistic (like in the case of Naive Bayes),
 \item geometric (like in SVM),
 \item information theoretic (entropy models).
\end{itemize}
We focus on the last group of approaches, and investigate the applicability of
the {\em Cauchy-Schwarz divergence}~\cite{jenssen2006cauchy}, which for 
two densities $f$ and $g$ is given by
\begin{equation*}
 \begin{aligned}
\Dcs(f,g)&=\ln \left (\int f^2 \right )+\ln \left (\int g^2 \right )-2\ln \left (\int fg \right)\\
&=-2\ln \left (\int \tfrac{f}{\|f\|_2} \tfrac{g}{\|g\|_2} \right ).
 \end{aligned}
\end{equation*}

Cauchy-Schwarz divergence is connected to Renyi's quadratic cross entropy ($H_2^\times$~\cite{principe}) and Renyi's
quadratic entropy ($H_2$), defined for densities $f,g$ as
\begin{equation*}
 \begin{aligned}
 H_2^\times(f,g)&=-\ln \left (\int fg \right ) \\
 H_2(f)=H_2^\times(f,f)&=-\ln \left (\int f^2 \right ),
 \end{aligned}
\end{equation*}
consequently
\begin{equation*}
 \begin{aligned}
\Dcs(f,g)=2H_2^\times(f,g) -H_2(f)-H_2(g).
 \end{aligned}
\end{equation*}
and as we showed in \cite{MELC},
it is well-suited as a discrimination measure which
allows the construction of mulit-threshold linear
classifiers. In general increase of the value of Cauchy-Schwarz
Divergence results in better sets' (densities') discrimination.
Unfortunately, there are a few problems with such an approach:
\begin{itemize}
 \item true datasets are discrete, so we do not have densities $f$ and $g$,
 \item statistical density estimators require rather large sample sizes and
       are very computationally expensive.
\end{itemize}

There are basically two approaches which help us recover underlying densities from the samples. First one is performing some kind of density esimation, like the well known Kernel Density Estimation (KDE) technique, which is based on the observation that any arbitrary continuous distribution can be sufficiently approximated by the convex combination of Gaussians.
The other approach is to assume some density model (distribution's family) and fit its parameters in order to maximize the data generation probability. In statistics it is known as maximum likelihood esetimation (MLE) approach. MLE has an advantage that in general it produces much simplier densities descriptions than KDE as later's description is linearly big in terms of sample size. 

A common choice of density models are Gaussian distributions due to their nice theoretical and practical (computational) capabilities. As mentioned eariler, the conxex combination of Gaussians can approximate the given continuous distribution $f$ with arbitrary precision. In order to fit a Gaussian Mixture Model (GMM) to given dataset one needs algorithm like Expectation Maximization~\cite{dempster1977maximum} or conceptually similar Cross-Entropy Clustering~\cite{tabor2014cross}. However, for simplicity and strong regularization we propose to model $f$ as one big Gaussian $\nor(m,\Sigma)$. One of the biggest advantages of such an approach is closed form MLE parameter estimation, as we simply put $m$ equal to the empirical mean of the data, and $\Sigma$ as some data covariance estimator. Secondly, this way we introduce an error to the data which has an important regularizing role and leads to better posed optimization problem.

Let us now recall that the Shannon's differential entropy (expressed in nits) of the continuous distribution $f$ is
$$
H(f) = - \int f \ln f,
$$
we will now show that choice of Normal distributions is not arbitrary but supported by the assumptions of the entropy maximization. Following result is known, but we include the whole reasoning for completeness.

\begin{remark}
 Normal distribution $\nor(m,\Sigma)$ has a maximum Shannon's differential entropy among all real-valued distributions with mean $m \in \R^h$ and covariance $\Sigma \in \R^{h \times h}$.
\end{remark}
\begin{proof}
 Let $f$ and $g$ be arbitrary distributions with covariance $\Sigma$ and means $m$. For simplicity we assume that $m=0$ but the analogous proof holds for arbitrary mean, then
 $$
 \int f \hpoint_i \hpoint_j d \hpoint_i d \hpoint_j = \int g \hpoint_i \hpoint_j d \hpoint_i d \hpoint_j = \Sigma_{ij}, 
 $$
 so for quadratic form $A$
 $$
 \int Af = \int Ag.
 $$
 Notice that 
 \begin{equation*}
  \begin{aligned}
 \ln \nor(0,\Sigma)[\hpoint] &= \ln \left ( \frac{1}{\sqrt{(2\pi)^h \det(\Sigma)}} \exp( -\tfrac{1}{2}\hpoint^T \Sigma^{-1} \hpoint) \right )\\
 &= -\frac{1}{2} \ln( (2\pi)^h \det(\Sigma) ) -\tfrac{1}{2}\hpoint^T \Sigma^{-1} \hpoint 
  \end{aligned}
 \end{equation*}
 is a quadratic form plus constant thus $$\int f \ln \nor(0,\Sigma) = \int \nor(0,\Sigma) \ln \nor(0,\Sigma),$$which together with non-negativity of Kullback-Leibler Divergence gives
 \begin{equation*}
  \begin{aligned}
   0 &\leq \Dkl(f\;||\;\nor(0,\Sigma)) \\
   &= \int f \ln ( \tfrac{f}{\nor(0,\Sigma)} ) \\
   &= \int f \ln f - \int f \ln \nor(0,\Sigma) \\
   &= -{H}(f) - \int f \ln \nor(0,\Sigma) \\
   &= -{H}(f) - \int \nor(0,\Sigma) \ln \nor(0,\Sigma) \\
   &= -{H}(f) + {H}( \nor(0,\Sigma) ),   
  \end{aligned}
 \end{equation*}
thus
$$
{H}( \nor(0,\Sigma) ) \geq {H}(f).
$$
 
\end{proof}

There appears nontrivial question how to find/estimate
the desired Gaussian as the covariance can be singular.
In this case to regularize the covariance 
we apply the well-known Ledoit-Wolf approach \cite{ledoit2004well}.

$$
\Sigma^{\pm}=\covlw ( \Hlayer^{\pm} ) = (1-\varepsilon^\pm)\cov(\Hlayer^\pm) + \varepsilon^\pm \mathrm{tr}(\cov(\Hlayer^\pm))h^{-1} I, 
$$
where $\cov(\cdot)$ is an empirical covariance and $\varepsilon^\pm$ is a shrinkage coefficient given by Ledoit and Wolf~\cite{ledoit2004well}.

%


\def\layersep{2.5cm}

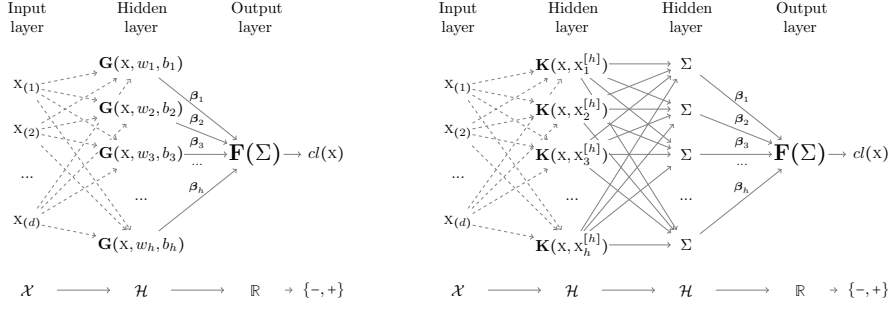
\begin{figure*}[h]
 \centering

\begin{minipage}[b]{0.2\linewidth}
\scalebox{.6}{
\begin{tikzpicture}[shorten >=1pt,->,draw=black!50, node distance=\layersep]
    \tikzstyle{every pin edge}=[<-,shorten <=1pt]
    \tikzstyle{neuron}=[fill=white, minimum size=17pt,inner sep=0pt]
    \tikzstyle{input neuron}=[neuron];
    \tikzstyle{output neuron}=[neuron];
    \tikzstyle{hidden neuron}=[neuron];
    \tikzstyle{annot} = [text width=3em, text centered]


	\node[input neuron] (I-1) at (0,-1) {$\xpoint_{(1)}$};
        \node[input neuron] (I-2) at (0,-2) {$\xpoint_{(2)}$};
        \node[input neuron] (I-3) at (0,-3) {$...$};
        \node[input neuron] (I-4) at (0,-4) {$\xpoint_{(d)}$};

        
        \path[yshift=0.5cm]
            node[hidden neuron] (H-1) at (\layersep,-1 cm) { $\Gfunction(\xpoint,w_1,b_1)$};
        \path[yshift=0.5cm]
            node[hidden neuron] (H-2) at (\layersep,-2 cm) { $\Gfunction(\xpoint,w_2,b_2)$};
        \path[yshift=0.5cm]
            node[hidden neuron] (H-3) at (\layersep,-3 cm) { $\Gfunction(\xpoint,w_3,b_3)$};
        \path[yshift=0.5cm]
            node[hidden neuron] (H-4) at (\layersep,-4 cm) { $...$};
        \path[yshift=0.5cm]
            node[hidden neuron] (H-5) at (\layersep,-5 cm) { $\Gfunction(\xpoint,w_h,b_h)$};

    \node[output neuron,pin={[pin edge={->}]right:$cl(\xpoint)$}, right of=H-3] (O) {\Large $\Ffunction(\Sigma)$};

    \foreach \source in {1,2,4}
        \foreach \dest in {1,2,3,5}
            \path[dash pattern=on 2pt off 2pt] (I-\source) edge (H-\dest) ;

        \path[yshift=0.5cm]
            node[hidden neuron] (H-1) at (\layersep,-1 cm) { $\Gfunction(\xpoint,w_1,b_1)$};
        \path[yshift=0.5cm]
            node[hidden neuron] (H-2) at (\layersep,-2 cm) { $\Gfunction(\xpoint,w_2,b_2)$};
        \path[yshift=0.5cm]
            node[hidden neuron] (H-3) at (\layersep,-3 cm) { $\Gfunction(\xpoint,w_3,b_3)$};
        \path[yshift=0.5cm]
            node[hidden neuron] (H-4) at (\layersep,-4 cm) { $...$};
        \path[yshift=0.5cm]
            node[hidden neuron] (H-5) at (\layersep,-5 cm) { $\Gfunction(\xpoint,w_h,b_h)$};

        \path[yshift=0.5cm]
            node[hidden neuron] (b-1) at (\layersep + 35,-35 + -1 * 0.5 cm) {\scriptsize $\LinearOperator_1$};
        \path[yshift=0.5cm]
            node[hidden neuron] (b-2) at (\layersep + 35,-35 + -2 * 0.5 cm) {\scriptsize $\LinearOperator_2$};
        \path[yshift=0.5cm]
            node[hidden neuron] (b-3) at (\layersep + 35,-35 + -3 * 0.5 cm) {\scriptsize $\LinearOperator_3$};
        \path[yshift=0.5cm]
            node[hidden neuron] (b-4) at (\layersep + 35,-35 + -4 * 0.5 cm) {\scriptsize $...$};
        \path[yshift=0.5cm]
            node[hidden neuron] (b-5) at (\layersep + 35,-35 + -5 * 0.5 cm) {\scriptsize $\LinearOperator_h$};

    \foreach \source in {1,2,3,5}
        \path (H-\source) edge (O);


    \node[annot,above of=H-1, node distance=1cm] (hl) {Hidden layer};
    \node[annot,left of=hl] {Input layer};
    \node[annot,right of=hl] {Output layer};
    
    \node[annot,below of=H-5, node distance=1cm] (xhl) {$\Hspace$};
    \node[annot,left  of=xhl] (xi)  {$\Xspace$};
    \node[annot,right of=xhl] (xo) {$\R$};
    \node[annot,right of=xo, node distance=1.5cm] (xc) {$\{-,+\}$};
    \path (xi) edge (xhl);
    \path (xhl) edge (xo);
    \path (xo) edge (xc);
\end{tikzpicture}
}
\end{minipage}
\hspace{3.01cm}
\begin{minipage}[b]{0.5\linewidth}
\scalebox{.6}{
\begin{tikzpicture}[shorten >=1pt,->,draw=black!50, node distance=\layersep]
    \tikzstyle{every pin edge}=[<-,shorten <=1pt]
    \tikzstyle{neuron}=[fill=white, minimum size=17pt,inner sep=0pt]
    \tikzstyle{input neuron}=[neuron];
    \tikzstyle{output neuron}=[neuron];
    \tikzstyle{hidden neuron}=[neuron];
    \tikzstyle{annot} = [text width=3em, text centered]


	\node[input neuron] (I-1) at (0,-1) {$\xpoint_{(1)}$};
        \node[input neuron] (I-2) at (0,-2) {$\xpoint_{(2)}$};
        \node[input neuron] (I-3) at (0,-3) {$...$};
        \node[input neuron] (I-4) at (0,-4) {$\xpoint_{(d)}$};

        
        \path[yshift=0.5cm]
            node[hidden neuron] (H-1) at (\layersep,-1 cm) { $\Kernel(\xpoint,\xpoint\subsample_1)$};
        \path[yshift=0.5cm]
            node[hidden neuron] (H-2) at (\layersep,-2 cm) { $\Kernel(\xpoint,\xpoint\subsample_2)$};
        \path[yshift=0.5cm]
            node[hidden neuron] (H-3) at (\layersep,-3 cm) { $\Kernel(\xpoint,\xpoint\subsample_3)$};
        \path[yshift=0.5cm]
            node[hidden neuron] (H-4) at (\layersep,-4 cm) { $...$};
        \path[yshift=0.5cm]
            node[hidden neuron] (H-5) at (\layersep,-5 cm) { $\Kernel(\xpoint,\xpoint\subsample_h)$};

	\path[yshift=0.5cm]
            node[hidden neuron] (H2-1) at (2*\layersep,-1 cm) { $\Sigma$};
        \path[yshift=0.5cm]
            node[hidden neuron] (H2-2) at (2*\layersep,-2 cm) { $\Sigma$};
        \path[yshift=0.5cm]
            node[hidden neuron] (H2-3) at (2*\layersep,-3 cm) { $\Sigma$};
        \path[yshift=0.5cm]
            node[hidden neuron] (H2-4) at (2*\layersep,-4 cm) { $...$};
        \path[yshift=0.5cm]
            node[hidden neuron] (H2-5) at (2*\layersep,-5 cm) { $\Sigma$};

    \node[output neuron,pin={[pin edge={->}]right:$cl(\xpoint)$}, right of=H2-3] (O) {\Large $\Ffunction(\Sigma)$};

    \foreach \source in {1,2,4}
        \foreach \dest in {1,2,3,5}
            \path[dash pattern=on 2pt off 2pt] (I-\source) edge (H-\dest) ;
    \foreach \source in {1,2,3,5}
        \foreach \dest in {1,2,3,5}
            \path (H-\source) edge (H2-\dest) ;

        \path[yshift=0.5cm]
            node[hidden neuron] (H-1) at (\layersep,-1 cm) { $\Kernel(\xpoint,\xpoint\subsample_1)$};
        \path[yshift=0.5cm]
            node[hidden neuron] (H-2) at (\layersep,-2 cm) { $\Kernel(\xpoint,\xpoint\subsample_2)$};
        \path[yshift=0.5cm]
            node[hidden neuron] (H-3) at (\layersep,-3 cm) { $\Kernel(\xpoint,\xpoint\subsample_3)$};
        \path[yshift=0.5cm]
            node[hidden neuron] (H-4) at (\layersep,-4 cm) { $...$};
        \path[yshift=0.5cm]
            node[hidden neuron] (H-5) at (\layersep,-5 cm) { $\Kernel(\xpoint,\xpoint\subsample_h)$};

	\path[yshift=0.5cm]
            node[hidden neuron] (H2-1) at (2*\layersep,-1 cm) { $\Sigma$};
        \path[yshift=0.5cm]
            node[hidden neuron] (H2-2) at (2*\layersep,-2 cm) { $\Sigma$};
        \path[yshift=0.5cm]
            node[hidden neuron] (H2-3) at (2*\layersep,-3 cm) { $\Sigma$};
        \path[yshift=0.5cm]
            node[hidden neuron] (H2-4) at (2*\layersep,-4 cm) { $...$};
        \path[yshift=0.5cm]
            node[hidden neuron] (H2-5) at (2*\layersep,-5 cm) { $\Sigma$};

        \path[yshift=0.5cm]
            node[hidden neuron] (b-1) at (2*\layersep + 35,-35 + -1 * 0.5 cm) {\scriptsize $\LinearOperator_1$};
        \path[yshift=0.5cm]
            node[hidden neuron] (b-2) at (2*\layersep + 35,-35 + -2 * 0.5 cm) {\scriptsize $\LinearOperator_2$};
        \path[yshift=0.5cm]
            node[hidden neuron] (b-3) at (2*\layersep + 35,-35 + -3 * 0.5 cm) {\scriptsize $\LinearOperator_3$};
        \path[yshift=0.5cm]
            node[hidden neuron] (b-4) at (2*\layersep + 35,-35 + -4 * 0.5 cm) {\scriptsize $...$};
        \path[yshift=0.5cm]
            node[hidden neuron] (b-5) at (2*\layersep + 35,-35 + -5 * 0.5 cm) {\scriptsize $\LinearOperator_h$};

    \foreach \source in {1,2,3,5}
        \path (H2-\source) edge (O);


    \node[annot,above of=H-1, node distance=1cm] (hl) {Hidden layer};
    \node[annot,above of=H2-1, node distance=1cm] (hl2) {Hidden layer};
    \node[annot,left of=hl] {Input layer};
    \node[annot,right of=hl2] {Output layer};
    
    \node[annot,below of=H-5, node distance=1cm] (xhl) {$\Hspace$};
    \node[annot,below of=H2-5, node distance=1cm] (xhl2) {$\Hspace$};
    \node[annot,left  of=xhl] (xi)  {$\Xspace$};
    \node[annot,right of=xhl2] (xo) {$\R$};
    \node[annot,right of=xo, node distance=1.5cm] (xc) {$\{-,+\}$};
    \path (xi) edge (xhl);
    \path (xhl) edge (xhl2);
    \path (xhl2) edge (xo);
    \path (xo) edge (xc);
\end{tikzpicture}
}
\end{minipage}

\caption{Extreme Entropy Machine (on the left) and Extreme Entropy Kernel Machine (on the right) as neural networks. In both cases all weights are either randomly selected (dashed) or are the result of closed-form optimization (solid). }
\label{fig:nn}
\end{figure*}

Thus, our optimization problem can be stated as
follows:

\begin{problem}[Optimization problem]
Suppose that we are given two datasets
$\Hlayer^\pm$ in a Hilbert space $\Hspace $ which come from the Gaussian distributions $\nor(\meanpoint^\pm,\Sigma^\pm)$. 
Find $\LinearOperator \in \Hspace $ such that the datasets
$$
\LinearOperator^T \Hlayer^+ \mbox{ and } \LinearOperator^T \Hlayer^-
$$
are optimally discriminated in terms of Cauchy-Schwarz Divergence.
\end{problem}


Because $\Hlayer^\pm$ has density $\nor(\meanpoint^\pm,\Sigma^\pm)$,  $\LinearOperator^T\Xlayer^\pm$ has the density $\nor(\LinearOperator^T\meanpoint^\pm,\LinearOperator^T\Sigma^\pm \LinearOperator)$.
We put 
\begin{equation} \label{eq:S}
\meanpoint_\pm=\LinearOperator^T\meanpoint^\pm, \, S_\pm=\LinearOperator^T\Sigma^\pm \LinearOperator.
\end{equation}
Since, as one can easily compute~\cite{ckrbf},
$$
\int \frac{\nor(\meanpoint_+,S_+)}{\|\nor(\meanpoint_+,S_+)\|_2} \cdot \frac{\nor(\meanpoint_-,S_-)}{\|\nor(\meanpoint_-,S_-)\|_2}
$$
$$ 
=\frac{\nor(\meanpoint_+-\meanpoint_-,S_++S_-)[0]}{(\nor(0,2S_+)[0]\nor(0,2S_-)[0])^{1/2}} 
$$
$$
=\frac{(2\pi S_+ S_-)^{1/4}}{(S_++S_-)^{1/2}}\exp \left (-\frac{(\meanpoint_+-\meanpoint_-)^2}{2(S_++S_-)} \right ),
$$
we obtain that
\begin{equation}
\begin{aligned}
\Dcs&(\nor(\meanpoint_+,S_+),\nor(\meanpoint_-,S_-))\\
&=-\ln \left ( \int \frac{\nor(\meanpoint_+,S_+)}{\|\nor(\meanpoint_+,S_+)\|_2} \cdot \frac{\nor(\meanpoint_-,S_-)}{\|\nor(\meanpoint_-,S_-)\|_2} \right )  \\
&=-\frac{1}{2}\ln \frac{\pi}{2}-\ln \frac{\tfrac{1}{2}(S_++S_-)}{\sqrt{S_+S_-}}+\frac{(\meanpoint_+-\meanpoint_-)^2}{S_++S_-}.
\end{aligned}
\end{equation}
Observe that in the above equation the first term is constant, the second is the logarithm of the quotient of arithmetical and geometrical means (and therefore in the typical cases is bounded and close to zero). Consequently, crucial information is given by the last
term. In order to confirm this claim we perform experiments on over 20 datasets used in further evaluation (more details are located in the Evaluation section). We compute the Spearman's rank correlation coefficient between the $\Dcs(\nor(\meanpoint_+,S_+),\nor(\meanpoint_-,S_-))$ and $\frac{(\meanpoint_+-\meanpoint_-)^2}{S_++S_-}$ for hundread random projections to $\Hspace $ and hundread random linear operators $\LinearOperator$.
\begin{table}[H]
\centering
  \caption{Spearman's rank correlation coefficient between optimized term and whole $\Dcs$ for all datasets used in evaluation. Each column represents different dimension of the Hilbert space.}
  \label{tab:dcs}
 \begin{tabular}{lccccc}
 \toprule
 dataset & $1$ & $10$ & $100$ & $200$ & $500$ \\
 \midrule
 \textsc{ australian } & 0.928
&
-0.022
&
0.295
&
0.161
&
0.235
\\
\textsc{ breast-cancer } & 0.628
&
0.809
&
0.812
&
0.858
&
0.788
\\
\textsc{ diabetes } & -0.983
&
-0.976
&
-0.941
&
-0.982
&
-0.952
\\
\textsc{ german.numer } & 0.916
&
0.979
&
0.877
&
0.873
&
0.839
\\
\textsc{ heart } & 0.964
&
0.829
&
0.931
&
0.91
&
0.893
\\
\textsc{ ionosphere } & 0.999
&
0.988
&
0.98
&
0.978
&
0.984
\\
\textsc{ liver-disorders } & 0.232
&
0.308
&
0.363
&
0.33
&
0.312
\\
\textsc{ sonar } & -0.31
&
-0.542
&
-0.41
&
-0.407
&
-0.381
\\
\textsc{ splice } & -0.284
&
-0.036
&
-0.165
&
-0.118
&
-0.101
\\
\midrule
\textsc{ abalone7 } & 1.0
&
0.999
&
0.999
&
0.999
&
0.998
\\
\textsc{ arythmia } & 1.0
&
1.0
&
0.999
&
1.0
&
1.0
\\
\textsc{ balance } & 1.0
&
0.998
&
0.998
&
0.999
&
0.998
\\
\textsc{ car evaluation } & 1.0
&
0.998
&
0.998
&
0.997
&
0.997
\\
\textsc{ ecoli } & 0.964
&
0.994
&
0.995
&
0.998
&
0.995
\\
\textsc{ libras move } & 1.0
&
0.999
&
0.999
&
1.0
&
1.0
\\
\textsc{ oil spill } & 1.0
&
1.0
&
1.0
&
1.0
&
1.0
\\
\textsc{ sick euthyroid } & 1.0
&
0.999
&
1.0
&
1.0
&
1.0
\\
\textsc{ solar flare } & 1.0
&
1.0
&
1.0
&
1.0
&
1.0
\\
\textsc{ spectrometer } & 1.0
&
1.0
&
0.999
&
0.999
&
0.999
\\
\midrule
\textsc{ forest cover } & 0.988
&
0.997
&
0.997
&
0.992
&
0.988
\\
\textsc{ isolet } & 0.784
&
1.0
&
0.997
&
0.997
&
0.999
\\
\textsc{ mammography } & 1.0
&
1.0
&
1.0
&
1.0
&
1.0
\\
\textsc{ protein homology } & 1.0
&
1.0
&
1.0
&
1.0
&
1.0
\\
\textsc{ webpages } & 1.0
&
1.0
&
1.0
&
0.999
&
0.999
\\
\bottomrule

 \end{tabular}

\end{table}
As one can see in Table~\ref{tab:dcs}, in small datasets (first part of the table) the correlation is generally high, with some exceptions (like \textsc{sonar}, \textsc{splice}, \textsc{liver-disorders} and \textsc{diabetes}). However, for bigger datasets (consisting of thousands examples) this correlation is nearly perfect (up to the randomization process it is nearly $1.0$ for all cases) which is a very strong empirical confirmation of our claim that maximization of the $\frac{(\meanpoint_+-\meanpoint_-)^2}{S_++S_-}$ is generally equivalent to the maximization of $\Dcs(\nor(\meanpoint_+,S_+),\nor(\meanpoint_-,S_-))$.

This means that, after the above reductions, 
and application of \eqref{eq:S}
our final problem can be stated as follows:

%

\medskip

\noindent{\textbf{Optimization problem: Extreme Entropy Machine} }
\begin{equation*}
\begin{aligned}
& \underset{\LinearOperator}{\text{minimize}} 
& & \LinearOperator^T \Sigma^+ \LinearOperator + \LinearOperator^T \Sigma^- \LinearOperator &\\
& \text{subject to}
& & \LinearOperator^T( \meanpoint^+ - \meanpoint^- ) = 2\\
& \text{where}
& & \Sigma^\pm = \covlw (\Hlayer^\pm) \\
& & & \meanpoint^\pm = \frac{1}{|\Hlayer^\pm|} \sum_{\hpoint^\pm \in \Hlayer^\pm} \hpoint^\pm\\
& & & \Hlayer^\pm = \varphi(\Xlayer^\pm)
\end{aligned}
\end{equation*}

\medskip

Before we continue to the closed-form solution we outline two methods of actually transforming our data $\Xlayer^\pm \subset \Xspace $ to the 
highly dimensional $\Hlayer^\pm \subset \Hspace $, given by the $\varphi : \Xspace  \to \Hspace $.

We investigate two approaches which lead to the Extreme Entropy Machine and Extreme Entropy Kernel Machine respectively.

\begin{itemize}
 \item for \textbf{Extreme Entropy Machine} (EEM) we use the random projection technique, exactly the same as the one used in the ELM. In other words, given some generalized activation function $\Gfunction(\xpoint,w,b) : \Xspace  \times \Xspace  \times \mathbb{R} \to \R$ and a constant $h$ denoting number of hidden neurons: 
$$
\varphi : \Xspace \ni \xpoint \to [\Gfunction(\xpoint,w_1,b_1),\dots,\Gfunction(\xpoint,w_h,b_h)]^T \in \R^h  
$$
where $w_i$ are random vectors and $b_i$ are random biases.
\item for \textbf{Extreme Entropy Kernel Machine} (EEKM) we use the randomized kernel approximation technique~\cite{drineas2005nystrom}, which spans our Hilbert space on randomly selecteed subset of training vectors. In other words, given valid kernel $\Kernel(\cdot,\cdot) : \Xspace  \times \Xspace  \to \R_+$ and size of the kernel space base $h$:
$$
\varphi_\Kernel : \Xspace \ni \xpoint \to (\Kernel(\xpoint,\Xlayer\subsample)\Kernel(\Xlayer\subsample,\Xlayer\subsample)^{-1/2})^T \in \R^h  
$$
where $\Xlayer\subsample$ is a $h$ element random subset of $\Xlayer$. It is easy to verify that such low rank approxmation truly behaves as a kernel, in the sense that for $\varphi_\Kernel(\xpoint_i), \varphi_\Kernel(\xpoint_j) \in \R^{h} $
\begin{equation*}
 \begin{aligned}
  \varphi_\Kernel(\xpoint_i)^T&\varphi_\Kernel(\xpoint_j) = \\
  =  &((\Kernel(\xpoint_i,\Xlayer\subsample)\Kernel(\Xlayer\subsample,\Xlayer\subsample)^{-1/2})^T)^T \\
    &( \Kernel(y,\Xlayer\subsample)\Kernel(\Xlayer\subsample,\Xlayer\subsample)^{-1/2} )^T \\
  =  &\Kernel(\xpoint_i,\Xlayer\subsample)\Kernel(\Xlayer\subsample,\Xlayer\subsample)^{-1/2} \\
    &( \Kernel(y,\Xlayer\subsample)\Kernel(\Xlayer\subsample,\Xlayer\subsample)^{-1/2} )^T \\
  = &\Kernel(\xpoint_i,\Xlayer\subsample)\Kernel(\Xlayer\subsample,\Xlayer\subsample)^{-1/2}\\
    &\Kernel(\Xlayer\subsample,\Xlayer\subsample)^{-1/2} \Kernel^T(\xpoint_j,\Xlayer\subsample) \\
  = &\Kernel(\xpoint_i,\Xlayer\subsample)\Kernel(\Xlayer\subsample,\Xlayer\subsample)^{-1} \Kernel(\Xlayer\subsample,\xpoint_j), 
 \end{aligned}
\end{equation*}
\noindent given true kernel projection $\phi_\Kernel$ such that $$\Kernel(\xpoint_i,\xpoint_j)=\phi_\Kernel(\xpoint_i)^T\phi_\Kernel(\xpoint_j)$$ we have  
\begin{equation*}
 \begin{aligned}
  \Kernel(\xpoint_i,\Xlayer\subsample)&\Kernel(\Xlayer\subsample,\Xlayer\subsample)^{-1} \Kernel(\Xlayer\subsample,\xpoint_j) = \\
  =&\phi_\Kernel(\xpoint_i)^T\phi_\Kernel(\Xlayer\subsample) \\
   &(\phi_\Kernel(\Xlayer\subsample)^T\phi_\Kernel(\Xlayer\subsample))^{-1}\\
   &\phi_\Kernel(\Xlayer\subsample)^T\phi_\Kernel(\xpoint_j) \\
  =&\phi_\Kernel(\xpoint_i)^T\phi_\Kernel(\Xlayer\subsample) \phi_\Kernel(\Xlayer\subsample)^{-1}\\
   &(\phi_\Kernel(\Xlayer\subsample)^T)^{-1} \phi_\Kernel(\Xlayer\subsample)^T\phi_\Kernel(\xpoint_j) \\
  =& \phi_\Kernel(\xpoint_i)^T\phi_\Kernel(\xpoint_j)\\
  =& \Kernel(\xpoint_i,\xpoint_j).
 \end{aligned}
\end{equation*}
Thus for the whole samples' set we have
$$
\varphi_\Kernel(\Xlayer)^T \varphi_\Kernel(\Xlayer) = \Kernel(\Xlayer,\Xlayer),
$$
which is a complete Gram matrix.
\end{itemize}

So the only difference between Extreme Entropy Machine and Extreme Entropy Kernel Machine is that in later we use $\Hlayer^\pm=\varphi_\Kernel(\Xlayer^\pm)$ where $\Kernel$ is a selected kernel instead of $\Hlayer^\pm=\varphi(\Xlayer^\pm)$. Fig.~\ref{fig:nn} visualizes these two approaches as neural networks, in particular EEM is a simple SLFN, while EEKM leads to the network with two hidden layers.

\begin{figure*}[h]
 \resizebox{\linewidth}{!}{
 \centering
\begin{tikzpicture}[      
        every node/.style={anchor=south west,inner sep=0pt},
        x=1mm, y=1mm,
      ]  
    \node (f2) at (38,-4) {\includegraphics[width=5cm]{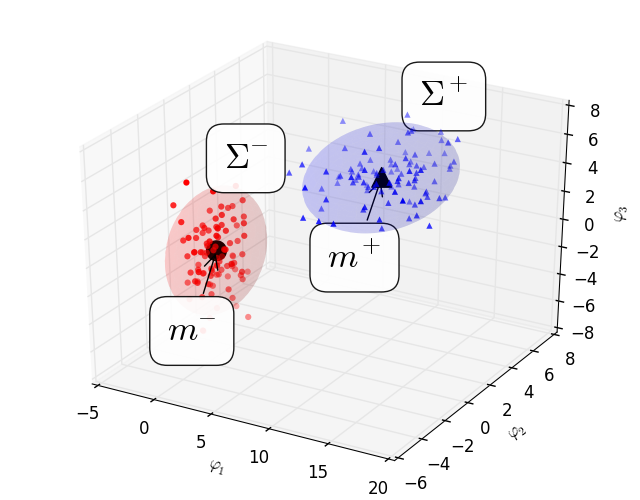}};
    \node (f1) at (0,0) {\includegraphics[width=4cm]{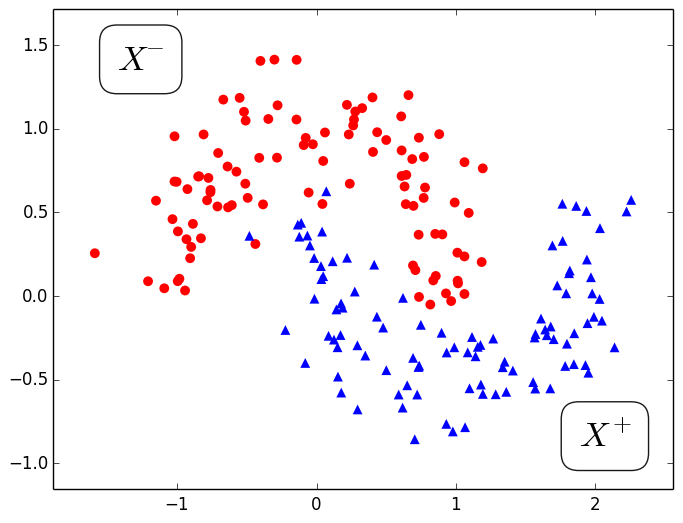}};
    \node (f3) at (90,0) {\includegraphics[width=4cm]{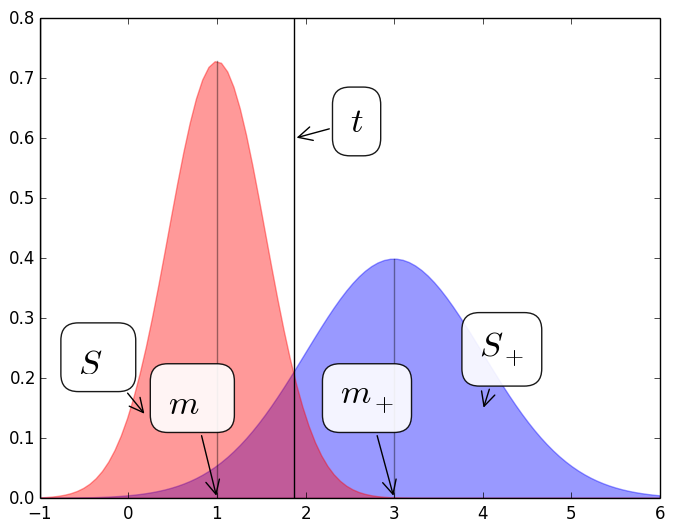}};
    \node (f4) at (135,0) {\includegraphics[width=4cm]{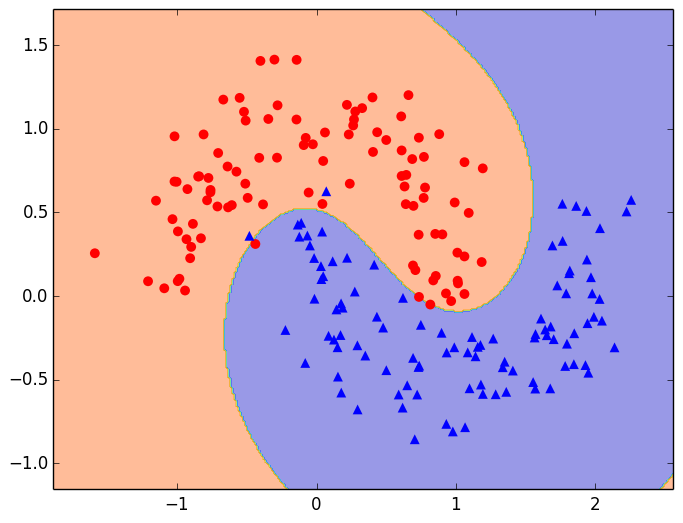}};

    \draw[thick,->] (36,20) -- (50,20) node[anchor=north] {};
    \node (p) at (41,23) {$\varphi$};

    \draw[thick,->] (78,20) -- (96,20) node[anchor=north] {};
    \node (p) at (86,23) {$\LinearOperator$};
    
    \draw[thick,->] (125,20) -- (142,20) node[anchor=north] {};
    \node (p) at (132,23) {$cl$};
    
\end{tikzpicture}
}

\caption{Visualization of the whole EEM classification process. From the left: Linearly non separable data in $\Xspace$; data mapped to the $\Hspace$ space, where we find covariance estimators; density of projected Gaussians on which the decision is based; decision boundary in the input space $\Xspace$.  }
\label{fig:classificaion}

\end{figure*}
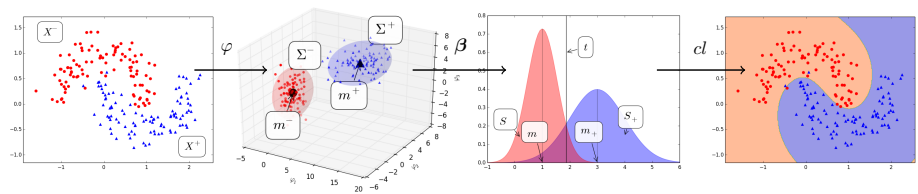

\begin{remark}
Extreme Entropy Machine optimization problem is closely related to the SVM optimization, but instead of maximizing the margin between closest points we are maximizing the mean margin.
\end{remark}
\begin{proof}
Let us recall that in SVM we try to maximize the margin $\frac{2}{\| \LinearOperator \|}$ under constraints that negative samples are projected at values at most -1 ($\LinearOperator^T \hpoint^- + b\leq -1 $) and positive samples on at least 1 ($\LinearOperator^T \hpoint^+ + b\geq 1$)
In other words, we are minimizing the $\LinearOperator$ operator norm $$\| \LinearOperator \|$$which is equivalent to minimizing the square of this norm $\| \LinearOperator\|^2$, under constraint that $$\min_{\hpoint^+ \in \Hlayer^+} \{ \LinearOperator^T \hpoint^+  \} -  \max_{\hpoint^- \in \Hlayer^-} \{ \LinearOperator^T \hpoint^-  \} = 1 - (-1) = 2.$$
On the other hand, EEM tries to minimize
\begin{equation*}
 \begin{aligned}
\LinearOperator^T \Sigma^+ \LinearOperator + \LinearOperator^T \Sigma^- \LinearOperator &= \LinearOperator^T ( \Sigma^+ + \Sigma^- ) \LinearOperator \\
&= \| \LinearOperator \|_{\Sigma^+ + \Sigma^-}^2
 \end{aligned}
\end{equation*}
under the constraint that
$$
\tfrac{1}{|\Hlayer^+|}\sum_{\hpoint^+ \in \Hlayer^+}  \LinearOperator^T \hpoint^+  -  \tfrac{1}{|\Hlayer^-|} \sum_{\hpoint^- \in \Hlayer^-}  \LinearOperator^T \hpoint^-   = 2.
$$
So what is happening here is that we are trying to maximize the mean margin between classes in the Mahalanobis norm 
generated by the sum of classes' covariances. It was previously shown in Two ellipsoid Support Vector Machines model~\cite{czarnecki2014two} that such norm is an approximation of the margin coming from two ellpisoids instead of the single ball used by traditional SVM.
\end{proof}

Similar observation regarding connection between large margin classification and entropy optimization has been done in case of the Multithreshold Linear Entropy Classifier~\cite{MELC}.


We are going to show by applying the standard
method of Lagrange multipliers that the above
problem has a closed form solution (similar to the Fischer's Discriminant).
Let
$$
\Sigma=\Sigma^++\Sigma^- \mbox{ and }\meanpoint=\meanpoint^+-\meanpoint^-.
$$
We put 
$$
L(\LinearOperator,\lambda):=2\LinearOperator^T\Sigma \LinearOperator-\lambda(\LinearOperator^T\meanpoint-2).
$$
Then
$$
\nabla_v L=2\Sigma \LinearOperator-\lambda \meanpoint
\text{ and } \frac{\partial}{\partial \lambda}L=\LinearOperator^T\meanpoint-2,
$$
which means that we need to solve, with respect to $\LinearOperator$, the system
$$
\begin{cases}
2\Sigma \LinearOperator-\lambda \meanpoint=0, \\
\LinearOperator^Tm=2.
\end{cases}
$$
Therefore $\LinearOperator=\frac{\lambda}{2}\Sigma^{-1}\meanpoint$, which
yields 
$$
\tfrac{\lambda}{2}\meanpoint^T\Sigma^{-1}\meanpoint=2,
$$
and consequently\footnote{where $\|\meanpoint\|^2_\Sigma=\meanpoint^T\Sigma^{-1}\meanpoint$ denotes
the squared Mahalanobis norm of $\meanpoint$.}, if $\meanpoint\neq 0$, then
$\lambda=4/\|\meanpoint\|^2_\Sigma$ and 
\begin{equation} \label{eq:final}
\begin{aligned}
 \LinearOperator&=\tfrac{2}{\|\meanpoint\|^2_\Sigma}\Sigma^{-1}\meanpoint\\
		&=\frac{2(\Sigma^+ + \Sigma^-)^{-1}(\meanpoint^+-\meanpoint^-)}{\|\meanpoint^+-\meanpoint^-\|^2_{\Sigma^+ + \Sigma^-}}.
\end{aligned}
\end{equation}

The final decision of the class of the
point $\hpoint$ is therefore given by 
the comparison of the values
$$
\nor(\LinearOperator^T\meanpoint^+,\LinearOperator^T\Sigma^+\LinearOperator)[\LinearOperator^T\hpoint] \text{ and }
\nor(\LinearOperator^T\meanpoint^-,\LinearOperator^T\Sigma^-\LinearOperator)[\LinearOperator^T\hpoint].
$$


We distinguish two cases based on number of resulting classifier's thresholds (points $t$ such that $\nor(\LinearOperator^T\meanpoint^+,\LinearOperator^T\Sigma^+\LinearOperator)[t] = \nor(\LinearOperator^T\meanpoint^-,\LinearOperator^T\Sigma^-\LinearOperator)[t]$):

\begin{enumerate}
 \item $S_- = S_+$, then there is one threshold  
\begin{flalign*}  
&t_0=\meanpoint_- + 1, &
\end{flalign*}
which results in a traditional (one-threshold) linear classifier,
 \item $S_- \neq S_+$, then there are two thresholds 
\begin{flalign*}  
&t_\pm = \meanpoint_- + \tfrac{2S_- \pm \sqrt{S_-S_+(\ln(S_-/S_+)(S_--S_+)+4)}}{S_--S_+}, &
\end{flalign*}
which makes the resulting classifier a member of two-thresholds linear classifiers family~\cite{anthony2003learning}.

\end{enumerate}
Obviously, in the degenerated case, when $m=0 \iff \meanpoint^-=\meanpoint^+$ there is no solution, as the constraint $\LinearOperator^T(\meanpoint^--\meanpoint^+)=2$ is not fulfilled for any $\LinearOperator$. In such a case EEM returna a trivial classifier constantly equal to any class (we put $\LinearOperator=0$).



From the neural network perspetive we simply construct a custom activation function $\Ffunction(\cdot)$ in the output neuron depending on one of the two described cases:
\begin{enumerate}
 \item $\Ffunction(x) = \left \{ \begin{matrix}
                                  +1, \text{ if } x \geq t_0\\
                                  -1, \text{ if } x < t_0
                                 \end{matrix}
 \right . = \textsc{sign}(x-t_0),$
 \item $\Ffunction(x) = \left \{ \begin{matrix}
                                  +1, \text{ if } x \in [t_-,t_+]\\
                                  -1, \text{ if } x \notin [t_-,t_+]
                                 \end{matrix}
 \right . = -\textsc{sign}(x-t_-) \textsc{sign}(x-t_+) ,$
 
 if $t_-<t_+$ and \\
 $\Ffunction(x) = \left \{ \begin{matrix}
                                  -1, \text{ if } x \in [t_+,t_-]\\
                                  +1, \text{ if } x \notin [t_+,t_-]
                                 \end{matrix}
 \right .  = \textsc{sign}(x-t_-) \textsc{sign}(x-t_+) ,$\\
  otherwise.
\end{enumerate}
The whole classification process is visualized in Fig.~\ref{fig:classificaion}, we begin with data in the input space $\Xspace $, transform it into Hilbert space $\Hspace $ where we model them as Gaussians, then perform optimization leading to the projection on $\R$ through $\LinearOperator$ and perform densitiy based classification leading to non-linear decision boundary in $\Xspace $.
%

\section{Theory: density estimation in the kernel case}

To illustrate our reasoning, we consider a typical basic problem
concerning the density estimation.

\begin{problem}
Assume that we are given a finite data set $\Hlayer$ in a Hilbert space
$\Hspace $ generated by the unknown density $f$, and we want
to obtain estimate of $f$.
\end{problem}

Since the problem in itself is infinite dimensional typically the data would be linearly independent. Moreover, one usually can not obtain reliable density estimation - the most
we can hope is that after transformation by a linear functional
into $\R$, the resulting density will be well-estimated.

To simplify the problem assume therefore that we want
to find the desired density in the class of normal densities --
or equivalently that we are interested only in the estimation 
of the mean and covariance of $f$.

The generalization of the above problem is given by the following
problem:

\begin{problem}
Assume that we are given a finite data sets $\Hlayer^\pm$ in a Hilbert space
$\Hspace $ generated by the unknown densities $f^\pm$, and we want
to obtain estimate of the unknown densities.
\end{problem}

In general $\text{dim}(\Hspace ) = h \gg N$ which means that we have very sparse data in terms of Hilbert space. As a result, classical kernel density estimation (KDE) is not reliable source of information~\cite{parzen1962estimation}. In the absence of different tools we can however use KDE with very big kernel width in order to cover at least some general shape of the whole density.

\begin{remark}
 Assume that we are given a finite data sets $\Hlayer^\pm$ with means $\meanpoint^\pm$ and covariances $\Sigma^\pm$ in a Hilbert space
 $\Hspace $. If we conduct kernel density estimation using Gaussian
 kernel then, in a limiting case, 
 each class becomes a Normal distribution.
 $$
 \lim_{\sigma \to \infty} \| \de{ \Hlayer^\pm }_\sigma - \nor( \meanpoint^\pm, \sigma^2 \Sigma^\pm ) \|_2 = 0,
 $$
 where
 $$
 \de{ A }_\sigma = \tfrac{1}{|A|} \sum_{a \in A}\nor(a, \sigma^2 \cdot \cov(A))
 $$
\end{remark}

Proof of this remark is given by Czarnecki and Tabor~\cite{MELC} and means that if we perform a Gaussian kernel density estimation of our data with big kernel width (which is reasonable for small amount of data in highly dimensional space) then for big enough $\hat \sigma$ EEM is nearly optimal linear classifier in terms of estimated densities $$ \hat f^\pm = \nor( \meanpoint^\pm, \hat \sigma^2 \Sigma^\pm ) \approx \de{ \Hlayer^\pm }_{\hat \sigma}. $$

\begin{sidewaystable*}[ph!]
\centering

 \caption{comparison of considered classiifers. $|SV|$ denotes number of support vectors. Asterix denotes features which can be adde to a particular model by some minor modifications, but we compare here the base versiond of each model.}
 \label{tab:comparison}
 \begin{tabular}{ccccc}
 \hline
	     & ELM & SVM & LS-SVM &  EE(K)M \\
 \hline
optimization method  & Linear regression 		  & Quadratic Programming & Linear System  &  \cellcolor{gray!10}Fischer's Discriminant \\
nonlinearity & random projection & kernel & kernel &  \cellcolor{gray!10}random (kernel) projection\\
closed form  & yes & no & yes &  \cellcolor{gray!10}yes \\
balanced     & no* & no* & no* &  \cellcolor{gray!10}yes \\
regression   & yes & no* & yes &  \cellcolor{gray!10}no \\
criterion    & Mean Squared Error & Hinge loss  & Mean Squared Error & \cellcolor{gray!10} Entropy optimization \\
no. of thresholds & 1 & 1 & 1 &  \cellcolor{gray!10}1 or 2 \\ 
problem type & regression & classification & regression & \cellcolor{gray!10} classification \\
model learning & discriminative & discriminative & discriminative & \cellcolor{gray!10} generative \\
direct probability estimates & no & no & no & \cellcolor{gray!10}	 yes \\
training complexity & $\mathcal{O}(Nh^2)$ & $\mathcal{O}(N^3)$ & $\mathcal{O}(N^{2.34})$ & \cellcolor{gray!10} $\mathcal{O}(Nh^2)$ \\
resulting model complexity & $hd$ & $|SV|d$, $|SV|\ll N $ & $Nd+1$ & \cellcolor{gray!10}$hd+4$\\ 
memory requirements & $\mathcal{O}(Nd)$ &  $\mathcal{O}(Nd)$ & $\mathcal{O}(N^2)$ & \cellcolor{gray!10}$\mathcal{O}(Nd)$ \\ 
source of regularization & Moore-Penrose pseudoinverse & margin maximization & quadratic loss penalty term & \cellcolor{gray!10} Ledoit-Wolf estimator  \\ 
\hline
 \end{tabular}
\end{sidewaystable*}




Let us now investigate the probabilistic interpretation of EEM.
Under the assumption that $ \Hlayer^\pm \sim \nor(\meanpoint^\pm, \Sigma^\pm) $ we have the conditional probabilities

$$
\Prob(\hpoint|\pm) = \nor(\meanpoint^\pm,\Sigma^\pm )[\hpoint],
$$

\noindent so from Bayes rule we conclude that

\begin{equation*}
 \begin{aligned}
\Prob(\pm|\hpoint) &= \frac{\Prob(\hpoint|\pm)\Prob(\pm)}{\Prob(\hpoint)} \\
&\propto \nor(\meanpoint^\pm,\Sigma^\pm )[\hpoint] \Prob(\pm),  
 \end{aligned}
\end{equation*}

\noindent  where $\Prob(\pm)$ is a prior classes' distribution. In our case, due to the balanced nature (meaning that despite classes imbalance we maximize the balanced quality measure such as Averaged Accuracy) we have $\Prob(\pm)=1/2$.

\noindent But
$$
\Prob(\hpoint) = \sum_{\ypoint \in \{+,-\}} \Prob(\hpoint|\ypoint) ,
$$

\noindent so
$$
\Prob(\pm|\hpoint) = \frac{
	    \nor(\meanpoint^\pm,\Sigma^\pm )[\hpoint]
	  }
	  { 
	    \sum_{\ypoint \in \{+,-\}} \nor(\meanpoint^\ypoint,\Sigma^\ypoint )[\hpoint]
	  }.
$$

Furthermore it is easy to show that under the normality assumption, the resulting classifier is optimal in the Bayesian sense.

\begin{remark}
 If data in feature space comes from Normal distributions $\nor(\meanpoint^\pm,\Sigma^\pm)$ then $\LinearOperator$ given by EEM minimizes probability of missclassification. More strictly speaking, if we draw $\hpoint^+$ with probability $1/2$ from $\nor(\meanpoint^+,\Sigma^+)$ and $\hpoint^-$ with 1/2 from $\nor(\meanpoint^-,\Sigma^-)$ then for any $\alpha \in \R^h$
 $$
 \Prob( \mp| \LinearOperator^T \hpoint^\pm ) \leq \Prob( \mp| \alpha^T \hpoint^\pm ) 
 $$
\end{remark}

\section{Theory: learning capabilities}

First we show that under some simplifing assumptions, proposed method behaves as Extreme Learning Machine (or Weighted Extreme Learning Machine~\cite{zong2013weighted}). 
 
 Before proceeding further we would like to remark that there are two popular notations for projecting data onto hyperplanes. One, used in ELM model, assumes that $\Hlayer$ is a row matrix and $\LinearOperator$ is a column vector, which results in the projection's equation 	$\Hlayer \LinearOperator$. Second one, used in SVM and in our paper, assumes that both $\Hlayer$ and $\LinearOperator$ are column oriented, which results in the $\beta^T\Hlayer$ projection. In the following theorem we will show some duality between $\LinearOperator$ found by ELM and by EEM. In order to do so, we will need to change the notation during the proof, which will be indicated.

\begin{theorem}
 Let us assume that we are given an arbitrary, balanced\footnote{analogous result can be shown for unbalanced dataset and Weighted ELM with particular weighting scheme.} dataset $\{(\xpoint_i,\ypoint_i)\}_{i=1}^N$, $\xpoint_i \in \R^d, \ypoint_i \in \{-1,+1\}, |\Xlayer^-|=|\Xlayer^+|$ which can be perfectly learned by ELM with $N$ hidden neurons. If this dataset's points' image through random neurons $\Hlayer=\varphi(\Xlayer)$ is centered (points' images have 0 mean) and classes have homogenous covariances (we can assume that $\exists_{a_\pm \in \R_+} \cov(\Hlayer) = a_+\cov(\Hlayer^+) = a_-\cov(\Hlayer^-)$ then EEM with the same hidden layer will also learn this dataset perfectly (with 0 error).  
\end{theorem}
\begin{proof}
 In the first part of the proof we use the ELM notation.  Projected data is centered, so $\cov(\Hlayer) = \Hlayer^T\Hlayer$. ELM is able to learn this dataset perfectly, consequently $\Hlayer$ is invertible, thus also $\Hlayer^T\Hlayer$ is invertible, as a result $\covlw (\Hlayer)=\cov(\Hlayer)=\Hlayer^T\Hlayer$. We will now show that 
 $ 
 \exists_{a \in \R_+} \LinearOperator_{\text{ELM}} = a \cdot \LinearOperator_{\text{EEM}}.
 $
 First, let us recall that $\LinearOperator_{\text{ELM}} = \Hlayer^\dagger \Yset = \Hlayer^{-1}\Yset$ and $\LinearOperator_{\text{EEM}}=\frac{2(\Sigma^++\Sigma^-)^{-1}(\meanpoint^+-\meanpoint^-)}{\| \meanpoint^+-\meanpoint^- \|_{\Sigma^-+\Sigma^+}^2}$ where $\Sigma^\pm = \covlw (\Hlayer^\pm)$. Due to the assumption of geometric homogenity  $\LinearOperator_{\text{EEM}}=\frac{2}{\|\meanpoint^+-\meanpoint^-\|_{\Sigma}^2}(\frac{a_++a_-}{a_+a_-}\Sigma)^{-1}(\meanpoint^+-\meanpoint^-)$ , where $\Sigma = \covlw (\Hlayer)$. Therefore
 \begin{equation*}
  \begin{aligned}
    \LinearOperator_{\text{ELM}} &= \Hlayer^{-1}\Yset \\ 
	        &= (\Hlayer^T\Hlayer)^{-1}\Hlayer^T\Yset \\
		&= \covlw ^{-1}(\Hlayer)\Hlayer^T\Yset 
  \end{aligned}
 \end{equation*} 
 From now we change the notation back to the one used in this paper. 
 \begin{equation*}
  \begin{aligned} 	
		\LinearOperator_{\text{ELM}} &= \Sigma^{-1} \left (\sum_{\hpoint^+ \in \Hlayer^+} (+1 \cdot \hpoint^+) + \sum_{\hpoint^- \in \Hlayer^-} (-1 \cdot \hpoint^-) \right )\\
		&= \Sigma^{-1} \left (\sum_{\hpoint^+ \in \Hlayer^+} \hpoint^+ - \sum_{\hpoint^- \in \Hlayer^-} \hpoint^- \right )\\
		&= \Sigma^{-1}\frac{N}{2}(\meanpoint^+ - \meanpoint^-)\\
		&= \frac{N}{2} \frac{\|\meanpoint^+-\meanpoint^-\|_{\Sigma}^2}{2}\frac{a_++a_-}{a_+a_-} \LinearOperator_{\text{EEM}}\\
		&= a \cdot \LinearOperator_{\text{EEM}},
  \end{aligned}
 \end{equation*}
 for $a = \frac{N}{2} \frac{\|\meanpoint^+-\meanpoint^-\|_{\Sigma}^2}{2}\frac{a_++a_-}{a_+a_-} \in \R_+$. Again from homogenity we obtain just one equilibrium point, located in the $\LinearOperator_{\text{EEM}}^T(\meanpoint^+-\meanpoint^-)/2$ which results in the exact same classifier as the one given by ELM. This completes the proof.

\end{proof}

Similar result holds for EEKM and Least Squares Support Vector Machine.

\begin{theorem}
 Let us assume that we are given arbitrary, balanced\footnote{analogous result can be shown for unbalanced dataset and Balanced LS-SVM with particular weighting scheme.} dataset $\{(\xpoint_i,\ypoint_i)\}_{i=1}^N$, $\xpoint_i \in \R^d, \ypoint_i \in \{-1,+1\}, |\Xlayer^-|=|\Xlayer^+|$ which can be perfectly learned by LS-SVM. If dataset's points' images through Kernel induced projection $\varphi_\Kernel$ have homogenous classes' covariances (we can assume that $\exists_{a_\pm \in \R_+} \cov(\varphi_\Kernel(\Xlayer)) = a_+\cov(\varphi_\Kernel(\Xlayer^+)) = a_-\cov(\varphi_\Kernel(\Xlayer^-))$ then EEKM with the same kernel and $N$ hidden neurons will also learn this dataset perfectly (with 0 error).  
\end{theorem}
\begin{proof}

It is a direct consequence of the fact that with $N$ hidden neurons and honogenous classes projections covariances, EEKM degenerates to the kernelized Fischer Discriminant which, as Gestel et al. showed ~\cite{van2004benchmarking}, is equivalent to the solution of the Least Squares SVM.
\end{proof}

\section{Practical considerations}

We can formulate the whole EEM training as a very simple algorithm (see Alg.~\ref{alg:eem}).
\begin{algorithm}[H]
\caption{Extreme Entropy (Kernel) Machine}
\label{alg:eem}
\begin{algorithmic}
\STATE 
\STATE \textbf{\textsc{train}}$(\Xlayer^+,\Xlayer^-)$
\STATE \hspace{0.5cm}\textbf{build} $\varphi$ \textbf{using Algorithm~\ref{alg:phi}}
\STATE \hspace{0.5cm}$ \Hlayer^\pm \gets \varphi(\Xlayer^\pm) $
\STATE \hspace{0.5cm}$ \meanpoint^\pm \gets 1/|\Hlayer^\pm| \sum_{\hpoint^\pm \in \Hlayer^\pm} \hpoint^\pm $
\STATE \hspace{0.5cm}$ \Sigma^\pm \gets \covlw(\Hlayer^\pm) $
\STATE \hspace{0.5cm}$ \LinearOperator \gets 2\left ( \Sigma^+ + \Sigma^- \right )^{-1}(\meanpoint^+-\meanpoint^-) / \|\meanpoint^+-\meanpoint^-\|_{\Sigma^+ + \Sigma^-} $
\STATE \hspace{0.5cm}$ \Ffunction(x) = \argmax_{\ypoint \in \{+,-\}} \nor(\LinearOperator^T \meanpoint^\ypoint, \LinearOperator^T \Sigma^\ypoint \LinearOperator)[x]$ 
\STATE \hspace{0.5cm}\textbf{return}  $\LinearOperator, \varphi, \Ffunction$
\STATE 
\STATE \textbf{\textsc{predict}}$(\Xlayer)$
\STATE \hspace{0.5cm}\textbf{return}  $\Ffunction( \LinearOperator^T \varphi(\Xlayer) )$
\end{algorithmic}
\end{algorithm}

\begin{algorithm}[H]
\caption{$\varphi$ building}
\label{alg:phi}
\begin{algorithmic}
\STATE 
\STATE \textbf{\textsc{Extreme Entropy Machine}}$(\Gfunction,h)$
\STATE \hspace{0.5cm}\textbf{select randomly } $w_i,b_i$ \textbf{ for } $i \in \{1,...,h\}$
\STATE \hspace{0.5cm}$ \varphi(\xpoint) = [\Gfunction(\xpoint,w_1,b_1),...,\Gfunction(\xpoint,w_h,b_h)]^T $
\STATE \hspace{0.5cm}\textbf{return}  $\varphi$
\STATE 
\STATE \textbf{\textsc{Extreme Entropy Kernel Machine}}$(\Kernel,h,\Xlayer)$
\STATE \hspace{0.5cm}\textbf{select randomly } $\Xlayer\subsample \subset \Xlayer, |\Xlayer\subsample|=h$ 
\STATE \hspace{0.5cm}$ \Kernel\subsample \gets \Kernel(\Xlayer\subsample,\Xlayer\subsample)^{-1/2}$
\STATE \hspace{0.5cm}$ \varphi_\Kernel(\xpoint) = \Kernel\subsample \Kernel(\Xlayer\subsample, \xpoint) $
\STATE \hspace{0.5cm}\textbf{return}  $\varphi_\Kernel$
\end{algorithmic}
\end{algorithm}

Resulting model consists of three elements:
\begin{itemize}
 \item feature projection function $\varphi$,
 \item linear operator $\LinearOperator$,
 \item classification rule $\Ffunction$.
\end{itemize}
As described before, $\Ffunction$ can be further compressed to just one or two thresholds $t_\pm$ using equations from previous sections. Either way, complexity of the resulting model is linear in terms of hidden units and classification of the new point takes $\mathcal{O}(dh)$ time. 

During EEM training, the most expensive part of the algorithm is the computation of the covariance estimators and inversion of the sum of covariances. Even computation of the empirical covariance takes $\mathcal{O}(Nh^2)$ time so the total complexity of training, equal to $\mathcal{O}(h^3 + Nh^2) = \mathcal{O}(Nh^2)$, is acceptable. It is worth noting that training of the ELM also takes exactly $\mathcal{O}(Nh^2)$ time as it requires computation of $\Hlayer^T\Hlayer$ for $\Hlayer \in \R^{N \times h}$. Training of EEMK requires additional computation of the square root of the sampled kernel matrix inverse $\Kernel(\Xlayer\subsample,\Xlayer\subsample)^{-1/2}$ but as $\Kernel(\Xlayer\subsample,\Xlayer\subsample) \in \R^{h \times h}$ can be computed in $\mathcal{O}(dh^2)$ and both inverting and square rooting can be done in $\mathcal{O}(h^3)$ we obtain exact same asymptotical computational complexity as the one of EEM. Procedure of square rooting and inverting are both always possible as assuming that $\Kernel$ 
is a valid 
kernel 
in the Mercer's sense yields that $\Kernel(\Xlayer\subsample,\Xlayer\subsample)$ is 
strictly positive definite and thus invertible. Further comparision of EEM, ELM and SVM is summarized in Table~\ref{tab:comparison}.

Next aspect we would like to discuss is the cost sensitive learning.
EEMs are balanced models in the sense that they are trying to maximize the balanced quality measures (like Averaged Accuracy or \textsc{GMean}). However, in practical applications it might be the case that we are actually more interested in the positive class then the negative one (like in the medical applications). Proposed model gives a direct probability estimates of $\Prob(\LinearOperator^T\hpoint|\ypoint)$, which we can easily convert to the cost sensitive classifier by introducing the prior probabilities of each class. Directly from Bayes Theorem, given $\Prob(+)$ and $\Prob(-)$, we can label our new sample $\hpoint$ according to
$$
\Prob(\ypoint|\LinearOperator^T\hpoint) \propto \Prob(\ypoint)\Prob(\LinearOperator^T\hpoint|\ypoint),
$$
so if we are given costs $C_+, C_- \in \R_+$ we can use them as weighting of priors
$$
cl(\xpoint) = \argmax_{\ypoint \in \{-,+\}} \tfrac{C_y}{C_-+C_+} \Prob(\LinearOperator^T\hpoint|\ypoint).
$$

Let us now investigate the possible efficiency bottleneck. In EEKM, the classification of the new point $\hpoint$ is based on
\begin{equation*}
 \begin{aligned}
cl(\xpoint) &= \Ffunction( \LinearOperator^T \varphi_\Kernel(\xpoint) ) \\
&= \Ffunction( \LinearOperator^T (\Kernel(\xpoint,\Xlayer\subsample)\Kernel\subsample)^T ) \\ 
&= \Ffunction( \LinearOperator^T (\Kernel\subsample)^T \Kernel(\xpoint,\Xlayer\subsample)^T )\\
&= \Ffunction( (\Kernel\subsample \LinearOperator)^T \Kernel(\Xlayer\subsample, \xpoint) ).
 \end{aligned}
\end{equation*}
One can convert EEKM to the SLFN by putting:
\begin{equation*}
  \begin{aligned}
  \hat \varphi_\Kernel(\xpoint) &= \Kernel(\Xlayer\subsample, \xpoint)\\
  \hat \LinearOperator &= \Kernel\subsample \LinearOperator ,
 \end{aligned}
\end{equation*}
so the classification rule becomes
\begin{equation*}
 \begin{aligned}
cl(\xpoint) &= \Ffunction( \hat \LinearOperator \phantom{}^T \hat \varphi_\Kernel(\xpoint) ).
 \end{aligned}
\end{equation*}
This way complexity of the new point's classification is exactly the same as in the case of EEM and ELM (or any other SLFN).
%
%
%
%
%
%
%
%
%
%
%
%

\section{Evaluation}

For the evaluation purposes we implemented five methods, namely: Weighted Extreme Learning Machine (WELM~\cite{zong2013weighted}), Extreme Entropy Machine (EEM), Extreme Entropy Kernel Machine (EEKM), Least Squares Support Vector Machines (LS-SVM~\cite{suykens1999least}) and Support Vector Machines (SVM~\cite{Vapnik95}).

All methods but SVM were implemented using Python with use of the bleeding-edge versions of \textsc{numpy}~\cite{van2011numpy} and \textsc{scipy}~\cite{jones2001scipy} libraries  included in \textsc{anaconda}\footnote{\url{https://store.continuum.io/cshop/anaconda/}} for fair comparision. For SVM we used highly efficient \textsc{libSVM}~\cite{chang2011libsvm} library with bindings avaliable in \textsc{scikit-learn}~\cite{pedregosa2011scikit}. 
Random projection based methods (WELM, EEM) were tested using three following generalized activation functions $\Gfunction(\xpoint,w,b)$
\begin{itemize}
 \item sigmoid (\textsc{sig}): $\tfrac{1}{1+\exp(-\langle w,\xpoint \rangle + b)} $,
 \item normalized sigmoid (\textsc{nsig}): $  \tfrac{1}{1+\exp(-\langle w,\xpoint \rangle /d + b)} $,
 \item radial basis function (\textsc{rbf}): $  \exp(-b \| w - \xpoint \|^2 )$.
\end{itemize}
Random parameters (weights and biases) were selected from uniform distributions on $[0,1]$. Training of WELM was performed using Moore-Penrose pseudoinverse and of EEM using Ledoit-Wolf covariance estimator, as both are parameter less, closed form estimators of required objects.
For kernel methods (EEKM, LS-SVM, SVM) we used the Gaussian kernel (\textsc{rbf}) $ \Kernel_\gamma(\xpoint_i,\xpoint_j)=\exp(-\gamma \| \xpoint_i-\xpoint_j \|^2 )$.
In all methods requiring class balancing schemes (WELM, LS-SVM, SVM) we used balance weights $w_i$ equal to the ratio of bigger class and current class (so $\sum_{i=1}^N w_i \ypoint_i = 0$). 

Metaparameters of each model were fitted, performed grid search included: hidden layer size $h=50,100,250,500,1000$ (WELM, EEM, EEKM), Gaussian Kernel width $\gamma=10^{-10},\ldots,10^0$ (EEKM, LS-SVM, SVM), SVM regularization parameter $C=10^{-1},\ldots,10^{10}$ (LS-SVM, SVM).

Datasets' features were linearly scaled in order to have each feature in the interval $[0,1]$. No other data whitening/filtering was performed. All experiments were performed in repeated 10-fold stratified cross-validation.

We use $\textsc{GMean}$\footnote{$\textsc{GMean}(\text{TP,FP,TN,FN}) = \sqrt{\frac{\text{TP}}{\text{TP}+\text{FN}} \cdot \frac{\text{TN}}{\text{TN}+\text{FP}}}$.} (geometric mean of accuracy over positive and negative samples) 	as an evaluation metric. due to its balanced nature and usage in previous works regarding Weighted Extreme Learning Machines~\cite{zong2013weighted}. 

\begin{table}[ht]
\caption{Characteristics of used datasets }
\label{tab:data}
\begin{center}

 \begin{tabular}{lrrrrrrr}
 \toprule
 dataset & d  & $|\Xlayer^-|$& $|\Xlayer^+|$ \\ 
 \midrule
\textsc{australian} & 14 & 383 & 307 \\
\textsc{bank} & 4 & 762 & 610 \\
\textsc{breast cancer} & 9 & 444 & 239 \\
\textsc{diabetes} & 8 & 268 & 500 \\
\textsc{german numer} & 24 & 700 & 300 \\
\textsc{heart} & 13 & 150 & 120 \\
\textsc{liver-disorders} & 6 & 145 & 200 \\
\textsc{sonar} & 60 & 111 & 97 \\
\textsc{splice} & 60 & 483 & 517 \\

\midrule

\textsc{abalone7} & 10 & 3786 & 391 \\
\textsc{arythmia} & 261 & 427 & 25 \\
\textsc{car evaluation} & 21 & 1594 & 134 \\
\textsc{ecoli} & 7 & 301 & 35 \\
\textsc{libras move} & 90 & 336 & 24 \\
\textsc{oil spill} & 48 & 896 & 41 \\
\textsc{sick euthyroid} & 42 & 2870 & 293 \\
\textsc{solar flare} & 32 & 1321 & 68 \\
\textsc{spectrometer} & 93 & 486 & 45 \\

\midrule

\textsc{forest cover} & 54 & 571519 & 9493 \\
\textsc{isolet} & 617 & 7197 & 600 \\
\textsc{mammography} & 6 & 10923 & 260 \\
\textsc{protein homology} & 74 & 144455 & 1296 \\
\textsc{webpages} & 300 & 33799 & 981 \\

\bottomrule

\end{tabular}
\end{center}
\end{table}

\subsection{Basic \textsc{UCI} datasets}

We start our experiments with nine datasets coming from \textsc{\textsc{UCI} repository~\cite{UCI}}, namely \textsc{australian}, \textsc{breast-cancer}, \textsc{diabetes}, \textsc{german.numer}, \textsc{heart}, \textsc{ionosphere},  \textsc{liver-disorders}, \textsc{sonar} and \textsc{splice}, summarized in the Table~\ref{tab:data}. This datasets include rather balanced,
low dimensional problems.

On such data, EEM seems to perform noticably better than ELM when using RBF activation function (see Table~\ref{tab:uci}), and rather similar when using sigmoid one -- in such a scenario, for some datasets ELM achieves better results while for other EEM wins. 
\begin{sidewaystable*}[ph!]
\centering
  \caption{\textsc{GMean} on \textsc{UCI} datasets}
  \label{tab:uci}
 \begin{tabular}{ccccccccccccc}
 \toprule
 &WELM$_{\textsc{sig}}$& EEM$_{\textsc{sig}}$&WELM$_{\textsc{nsig}}$& EEM$_{\textsc{nsig}}$  &WELM$_{\textsc{rbf}}$   & EEM$_{\textsc{rbf}}$  &  LS-SVM$_{\textsc{rbf}}$ & EEKM$_{\textsc{rbf}}$ &  SVM$_{\textsc{rbf}}$\\
 \midrule
 
{ \textsc{ australian  } } &

\scriptsize
86.3
\tiny $\pm4.5$
&
\scriptsize
\textbf{
87.0
}
\tiny $\pm4.0$
&
\scriptsize
85.9
\tiny $\pm4.4$
&
\scriptsize
86.5
\tiny $\pm3.2$
&
\scriptsize
85.8
\tiny $\pm4.9$
&
\scriptsize
86.9
\tiny $\pm4.4$
&
\scriptsize
86.9
\tiny $\pm4.1$
&
\scriptsize
86.8
\tiny $\pm3.8$
&
\scriptsize
86.8
\tiny $\pm3.7$
\\
\midrule

{ \textsc{ breast-cancer  } } &

\scriptsize
96.9
\tiny $\pm1.7$
&
\scriptsize
97.3
\tiny $\pm1.2$
&
\scriptsize
97.6
\tiny $\pm1.5$
&
\scriptsize
97.4
\tiny $\pm1.2$
&
\scriptsize
96.6
\tiny $\pm1.8$
&
\scriptsize
97.3
\tiny $\pm1.1$
&
\scriptsize
97.6
\tiny $\pm1.3$
&
\scriptsize
\textbf{
97.8
}
\tiny $\pm1.1$
&
\scriptsize
96.8
\tiny $\pm1.7$
\\
\midrule

{ \textsc{ diabetes  } } &

\scriptsize
74.2
\tiny $\pm4.6$
&
\scriptsize
74.5
\tiny $\pm4.6$
&
\scriptsize
74.1
\tiny $\pm5.5$
&
\scriptsize
74.9
\tiny $\pm5.0$
&
\scriptsize
73.2
\tiny $\pm5.6$
&
\scriptsize
74.9
\tiny $\pm5.9$
&
\scriptsize
75.5
\tiny $\pm5.6$
&
\scriptsize
\textbf{
75.7
}
\tiny $\pm5.6$
&
\scriptsize
74.8
\tiny $\pm3.5$
\\
\midrule

{ \textsc{ german } } &

\scriptsize
68.8
\tiny $\pm6.9$
&
\scriptsize
71.3
\tiny $\pm4.1$
&
\scriptsize
70.7
\tiny $\pm6.1$
&
\scriptsize
72.4
\tiny $\pm5.4$
&
\scriptsize
71.1
\tiny $\pm6.1$
&
\scriptsize
72.2
\tiny $\pm5.7$
&
\scriptsize
73.2
\tiny $\pm4.5$
&
\scriptsize
72.9
\tiny $\pm5.3$
&
\scriptsize
\textbf{
73.4
}
\tiny $\pm5.4$
\\
\midrule

{ \textsc{ heart  } } &

\scriptsize
78.8
\tiny $\pm6.3$
&
\scriptsize
82.5
\tiny $\pm7.4$
&
\scriptsize
78.1
\tiny $\pm7.0$
&
\scriptsize
83.7
\tiny $\pm7.2$
&
\scriptsize
80.2
\tiny $\pm8.9$
&
\scriptsize
81.9
\tiny $\pm6.9$
&
\scriptsize
83.7
\tiny $\pm8.5$
&
\scriptsize
83.6
\tiny $\pm7.5$
&
\scriptsize
\textbf{
84.6
}
\tiny $\pm7.0$
\\
\midrule

{ \textsc{ ionosphere  } } &

\scriptsize
71.5
\tiny $\pm9.5$
&
\scriptsize
77.0
\tiny $\pm12.8$
&
\scriptsize
82.7
\tiny $\pm7.8$
&
\scriptsize
84.6
\tiny $\pm9.1$
&
\scriptsize
85.6
\tiny $\pm8.4$
&
\scriptsize
90.8
\tiny $\pm5.2$
&
\scriptsize
91.2
\tiny $\pm5.5$
&
\scriptsize
93.4
\tiny $\pm4.3$
&
\scriptsize
\textbf{
94.7
}
\tiny $\pm3.9$
\\
\midrule

{ \textsc{ liver-disorders  } } &

\scriptsize
68.1
\tiny $\pm8.0$
&
\scriptsize
68.6
\tiny $\pm8.9$
&
\scriptsize
66.3
\tiny $\pm8.2$
&
\scriptsize
62.1
\tiny $\pm8.1$
&
\scriptsize
67.2
\tiny $\pm5.9$
&
\scriptsize
71.4
\tiny $\pm7.0$
&
\scriptsize
71.1
\tiny $\pm8.3$
&
\scriptsize
70.2
\tiny $\pm6.9$
&
\scriptsize
\textbf{
72.3
}
\tiny $\pm6.2$
\\
\midrule

{ \textsc{ sonar  } } &

\scriptsize
66.7
\tiny $\pm10.1$
&
\scriptsize
70.1
\tiny $\pm11.5$
&
\scriptsize
80.2
\tiny $\pm7.4$
&
\scriptsize
78.3
\tiny $\pm11.2$
&
\scriptsize
83.2
\tiny $\pm6.9$
&
\scriptsize
82.8
\tiny $\pm5.2$
&
\scriptsize
86.5
\tiny $\pm5.4$
&
\scriptsize
\textbf{
87.0
}
\tiny $\pm7.5$
&
\scriptsize
83.0
\tiny $\pm7.1$
\\
\midrule

{ \textsc{ splice  } } &

\scriptsize
64.7
\tiny $\pm2.8$
&
\scriptsize
49.4
\tiny $\pm5.5$
&
\scriptsize
81.8
\tiny $\pm3.2$
&
\scriptsize
80.9
\tiny $\pm2.7$
&
\scriptsize
75.5
\tiny $\pm3.9$
&
\scriptsize
82.2
\tiny $\pm3.5$
&
\scriptsize
\textbf{
89.9
}
\tiny $\pm3.0$
&
\scriptsize
88.0
\tiny $\pm4.0$
&
\scriptsize
88.0
\tiny $\pm2.2$
\\
\bottomrule

 \end{tabular}
\end{sidewaystable*}
Results obtained for EEKM are comparable with those obtained by LS-SVM and SVM, in both cases proposed method achieves better results on about third of problems, on the third it draws and on a third it loses. This experiments can be seen as a proof of concept of the whole methodology, showing that it can be truly a reasonable alternative for existing models in some problems. It appears that contrary to ELM, proposed methods (EEM and EEKM) achieve best scores across all considered models in some of the datasets regardless of the used activation function/kernel (only Support Vector Machines and their least squares counterpart are competetitive in this sense).

\subsection{Highly unbalanced datasets}

In the second part we proceeded to the nine highly unbalanced datasets, summarized in the second part of the Table~\ref{tab:data}. Ratio between bigger and smaller class varies from $10:1$ to even $20:1$ which makes them really hard for unbalanced models. Obtained results (see Table~\ref{tab:unb}) resembles these obtained on \textsc{UCI} repository. We can see better results in about half of experiments if we fix a particular activation function/kernel (so we compare ELM$_x$ with EEM$_x$ and LS-SVM$_x$ with EEKM$_x$).
\begin{sidewaystable*}[ph!]
\centering
\caption{Highly unbalanced datasets}
\label{tab:unb}
 \begin{tabular}{ccccccccccccc}
 \toprule
 &WELM$_{\textsc{sig}}$& EEM$_{\textsc{sig}}$&WELM$_{\textsc{nsig}}$& EEM$_{\textsc{nsig}}$  &WELM$_{\textsc{rbf}}$   & EEM$_{\textsc{rbf}}$  &  LS-SVM$_{\textsc{rbf}}$ & EEKM$_{\textsc{rbf}}$ &  SVM$_{\textsc{rbf}}$\\
 \midrule
 
{ \textsc{ abalone7 } } &

\scriptsize
79.7
\tiny $\pm2.3$
&
\scriptsize
79.8
\tiny $\pm3.5$
&
\scriptsize
80.0
\tiny $\pm2.8$
&
\scriptsize
76.1
\tiny $\pm3.7$
&
\scriptsize
80.1
\tiny $\pm3.2$
&
\scriptsize
79.7
\tiny $\pm3.6$
&
\scriptsize
\textbf{
80.2
}
\tiny $\pm3.4$
&
\scriptsize
79.9
\tiny $\pm3.4$
&
\scriptsize
79.7
\tiny $\pm2.7$
\\
\midrule

{ \textsc{ arythmia } } &

\scriptsize
28.3
\tiny $\pm35.4$
&
\scriptsize
40.3
\tiny $\pm20.9$
&
\scriptsize
64.2
\tiny $\pm24.6$
&
\scriptsize
\textbf{
85.6
}
\tiny $\pm10.3$
&
\scriptsize
66.9
\tiny $\pm25.8$
&
\scriptsize
79.4
\tiny $\pm12.5$
&
\scriptsize
84.4
\tiny $\pm10.0$
&
\scriptsize
85.2
\tiny $\pm10.6$
&
\scriptsize
80.9
\tiny $\pm11.8$
\\
\midrule

{ \textsc{ car evaluation } } &

\scriptsize
99.1
\tiny $\pm0.3$
&
\scriptsize
98.9
\tiny $\pm0.4$
&
\scriptsize
99.0
\tiny $\pm0.3$
&
\scriptsize
97.9
\tiny $\pm0.6$
&
\scriptsize
99.0
\tiny $\pm0.3$
&
\scriptsize
98.5
\tiny $\pm0.3$
&
\scriptsize
99.5
\tiny $\pm0.2$
&
\scriptsize
99.2
\tiny $\pm0.3$
&
\scriptsize
\textbf{
100.0
}
\tiny $\pm0.0$
\\
\midrule

{ \textsc{ ecoli } } &

\scriptsize
86.9
\tiny $\pm6.5$
&
\scriptsize
88.3
\tiny $\pm7.1$
&
\scriptsize
86.9
\tiny $\pm6.8$
&
\scriptsize
88.6
\tiny $\pm6.9$
&
\scriptsize
86.4
\tiny $\pm7.0$
&
\scriptsize
88.8
\tiny $\pm7.2$
&
\scriptsize
89.2
\tiny $\pm6.3$
&
\scriptsize
\textbf{
89.4
}
\tiny $\pm6.9$
&
\scriptsize
88.5
\tiny $\pm6.2$
\\
\midrule

{ \textsc{ libras move } } &

\scriptsize
65.5
\tiny $\pm10.7$
&
\scriptsize
19.3
\tiny $\pm8.1$
&
\scriptsize
82.5
\tiny $\pm12.0$
&
\scriptsize
93.0
\tiny $\pm11.8$
&
\scriptsize
89.6
\tiny $\pm11.9$
&
\scriptsize
93.9
\tiny $\pm11.9$
&
\scriptsize
96.5
\tiny $\pm8.6$
&
\scriptsize
\textbf{
96.6
}
\tiny $\pm8.7$
&
\scriptsize
91.6
\tiny $\pm11.9$
\\
\midrule

{ \textsc{ oil spill } } &

\scriptsize
86.0
\tiny $\pm6.9$
&
\scriptsize
\textbf{
88.8
}
\tiny $\pm6.5$
&
\scriptsize
83.8
\tiny $\pm7.6$
&
\scriptsize
84.7
\tiny $\pm8.7$
&
\scriptsize
85.8
\tiny $\pm9.3$
&
\scriptsize
88.1
\tiny $\pm6.1$
&
\scriptsize
86.7
\tiny $\pm8.4$
&
\scriptsize
87.2
\tiny $\pm4.9$
&
\scriptsize
85.7
\tiny $\pm11.4$
\\
\midrule

{ \textsc{ sick euthyroid } } &

\scriptsize
88.1
\tiny $\pm1.7$
&
\scriptsize
87.9
\tiny $\pm2.4$
&
\scriptsize
88.5
\tiny $\pm2.1$
&
\scriptsize
81.7
\tiny $\pm2.7$
&
\scriptsize
89.1
\tiny $\pm1.9$
&
\scriptsize
88.2
\tiny $\pm2.4$
&
\scriptsize
89.5
\tiny $\pm1.7$
&
\scriptsize
89.3
\tiny $\pm1.9$
&
\scriptsize
\textbf{
90.9
}
\tiny $\pm2.0$
\\
\midrule

{ \textsc{ solar flare } } &

\scriptsize
60.4
\tiny $\pm16.8$
&
\scriptsize
63.7
\tiny $\pm12.9$
&
\scriptsize
61.3
\tiny $\pm10.8$
&
\scriptsize
67.4
\tiny $\pm9.0$
&
\scriptsize
60.3
\tiny $\pm14.8$
&
\scriptsize
68.9
\tiny $\pm9.3$
&
\scriptsize
67.3
\tiny $\pm8.8$
&
\scriptsize
67.3
\tiny $\pm9.0$
&
\scriptsize
\textbf{
70.9
}
\tiny $\pm8.5$
\\
\midrule

{ \textsc{ spectrometer } } &

\scriptsize
82.9
\tiny $\pm13.0$
&
\scriptsize
87.3
\tiny $\pm7.8$
&
\scriptsize
88.0
\tiny $\pm10.8$
&
\scriptsize
90.2
\tiny $\pm8.6$
&
\scriptsize
86.6
\tiny $\pm8.2$
&
\scriptsize
93.0
\tiny $\pm14.6$
&
\scriptsize
94.6
\tiny $\pm8.4$
&
\scriptsize
93.5
\tiny $\pm14.7$
&
\scriptsize
\textbf{
95.4
}
\tiny $\pm5.1$
\\
\bottomrule

 \end{tabular}

\end{sidewaystable*}
Table~\ref{tab:unbt} shows that training time of Extreme Entropy Machines are comparable with the ones obtained by Extreme Learning Machines (differences on the level of $0.1-0.2$ are not significant on such datasets' sizes). We have a robust method which learns in below two seconds a model for hundreads/thousands of examples. For larger datasets (like \textsc{abalone7} or \textsc{sick euthyroid}) proposed methods not only outperform SVM and LS-SVM in terms of robustness but there is also noticable difference between their training times and ELMs. This suggests that even though ELM and EEM are quite similar and on small datasets are equally fast, EEM can better scale up to truly big datasets. Obviously obtained training times do not resemble the full training time as it strongly depends on the technique used for metaparameters selection and resolution of grid search (or other parameters tuning technique). In such full scenario, training times of SVM 
related models is significantly bigger due to the requirment of exact tuning of both $C$ and $\gamma$ in real domains.  

\begin{sidewaystable*}[ph!]
\centering
\caption{ highly unbalanced datasets times}
\label{tab:unbt}
 \begin{tabular}{ccccccccccccc}
 \toprule
 &WELM$_{\textsc{sig}}$& EEM$_{\textsc{sig}}$ &WELM$_{\textsc{nsig}}$& EEM$_{\textsc{nsig}}$ &WELM$_{\textsc{rbf}}$   & EEM$_{\textsc{rbf}}$  &  LS-SVM$_{\textsc{rbf}}$ & EEKM$_{\textsc{rbf}}$ &  SVM$_{\textsc{rbf}}$\\
 \midrule
 
{ \textsc{ abalone7 } } &

\scriptsize
1.9
s
&
\scriptsize
1.2
s
&
\scriptsize
2.5
s
&
\scriptsize
1.6
s
&
\scriptsize
1.8
s
&
\scriptsize
\textbf{
1.2
}
s
&
\scriptsize
20.8
s
&
\scriptsize
1.9
s
&
\scriptsize
4.7
s
\\
\midrule

{ \textsc{ arythmia } } &

\scriptsize
0.2
s
&
\scriptsize
0.7
s
&
\scriptsize
0.3
s
&
\scriptsize
0.9
s
&
\scriptsize
0.3
s
&
\scriptsize
0.7
s
&
\scriptsize
\textbf{
0.1
}
s
&
\scriptsize
0.3
s
&
\scriptsize
0.1
s
\\
\midrule

{ \textsc{ car evaluation } } &

\scriptsize
1.3
s
&
\scriptsize
0.9
s
&
\scriptsize
1.5
s
&
\scriptsize
1.0
s
&
\scriptsize
1.1
s
&
\scriptsize
0.9
s
&
\scriptsize
2.0
s
&
\scriptsize
1.4
s
&
\scriptsize
\textbf{
0.1
}
s
\\
\midrule

{ \textsc{ ecoli } } &

\scriptsize
0.2
s
&
\scriptsize
0.8
s
&
\scriptsize
0.2
s
&
\scriptsize
0.8
s
&
\scriptsize
0.1
s
&
\scriptsize
0.7
s
&
\scriptsize
\textbf{
0.0
}
s
&
\scriptsize
0.1
s
&
\scriptsize
0.2
s
\\
\midrule

{ \textsc{ libras move } } &

\scriptsize
0.2
s
&
\scriptsize
0.9
s
&
\scriptsize
0.2
s
&
\scriptsize
0.8
s
&
\scriptsize
0.1
s
&
\scriptsize
0.7
s
&
\scriptsize
0.0
s
&
\scriptsize
0.1
s
&
\scriptsize
\textbf{
0.0
}
s
\\
\midrule

{ \textsc{ oil spill } } &

\scriptsize
0.7
s
&
\scriptsize
0.8
s
&
\scriptsize
0.6
s
&
\scriptsize
0.8
s
&
\scriptsize
0.6
s
&
\scriptsize
0.8
s
&
\scriptsize
0.4
s
&
\scriptsize
0.9
s
&
\scriptsize
\textbf{
0.1
}
s
\\
\midrule

{ \textsc{ sick euthyroid } } &

\scriptsize
1.5
s
&
\scriptsize
1.1
s
&
\scriptsize
1.4
s
&
\scriptsize
\textbf{
1.1
}
s
&
\scriptsize
1.5
s
&
\scriptsize
1.1
s
&
\scriptsize
9.6
s
&
\scriptsize
1.7
s
&
\scriptsize
21.0
s
\\
\midrule

{ \textsc{ solar flare } } &

\scriptsize
\textbf{
0.7
}
s
&
\scriptsize
0.8
s
&
\scriptsize
0.7
s
&
\scriptsize
0.8
s
&
\scriptsize
0.8
s
&
\scriptsize
0.8
s
&
\scriptsize
1.1
s
&
\scriptsize
1.3
s
&
\scriptsize
16.1
s
\\
\midrule

{ \textsc{ spectrometer } } &

\scriptsize
0.2
s
&
\scriptsize
0.7
s
&
\scriptsize
0.3
s
&
\scriptsize
0.7
s
&
\scriptsize
0.2
s
&
\scriptsize
0.7
s
&
\scriptsize
0.1
s
&
\scriptsize
0.3
s
&
\scriptsize
\textbf{
0.0
}
s
\\
\bottomrule

 \end{tabular}

\end{sidewaystable*}

\subsection{Extremely unbalanced datasets}

Third part of experiments consists of extremely unbalanced datasets (with class imbalance up to 100:1) containing tens and hundreads thousands of examples. Five analyzed datasets span from NLP tasks (\textsc{webpages}) through medical applications (\textsc{mammography}) to bioinformatics (\textsc{protein homology}). This type of datasets often occur in the true data mining which makes these results much more practical than the one obtained on small/balanced data.

0.0 scores on \textsc{Isolet} dataset (see Table~\ref{tab:big}) for sigmoid based random projections is a result of very high values ($\sim~200$) of $\langle \xpoint,w \rangle$ for all $\xpoint$, which results in $\Gfunction(\xpoint,w,b)=1$, so the whole dataset is reduced to the singleton $\{ [1,\ldots,1]^T \} \subset \R^h \subset \Hspace $ which obviously is not separable by any classifier, netither ELM nor EEM.

For other activation functions we see that EEM achieves sllightly worse results than ELM. On the other hand, scores of EEKM generally outperform the ones obtained by ELM and are very close to the ones obtained by well tuned SVM and LS-SVM. In the same time, EEM and EEKM were trained significantly faster, as Table~\ref{tab:bigt} shows, it was order of magnitude faster than SVM related models and even $1.5-2 \times$ faster than ELM. It seems that the Ledoit-Wolf covariance estimation computation with this matrices inversion is simply a faster operation (scales better) than computation of the Moore-Penrose pseudoinverse of the $\Hlayer^T\Hlayer$. Obviously one can alternate ELM training routine to the regularized one where instead of $(\Hlayer^T\Hlayer)^\dagger$ one computes $(\Hlayer^T\Hlayer + I/C)^{-1}$, but we are analyzing here parameter less approaches, while the analogous could be used for EEM in the form of $(\cov(\Xlayer^-)+\cov(\Xlayer^+) + I/C)^{-1}$ instead of computing Ledoit-Wolf estimator. In 
other words, in the parameter less 
training scenario,
 as described in this paper EEMs seems to scale better than ELMs while still obtaining similar classification results. In the same time EEKM obtains SVM-level results with orders of magnitude smaller training times. Both ELM and EEM could be transformed into regularization parameter based learning, but this is beyond the scope of this work.

\begin{sidewaystable*}[ph!]
\centering
\caption{Big highly unbalanced datasets}
\label{tab:big}
 \begin{tabular}{ccccccccccccc}
 \toprule
 &WELM$_{\textsc{sig}}$& EEM$_{\textsc{sig}}$ &WELM$_{\textsc{nsig}}$& EEM$_{\textsc{nsig}}$ &WELM$_{\textsc{rbf}}$   & EEM$_{\textsc{rbf}}$  &  LS-SVM$_{\textsc{rbf}}$ & EEKM$_{\textsc{rbf}}$ &  SVM$_{\textsc{rbf}}$\\
 \midrule
 
{ \textsc{ forest cover } } &

\scriptsize
90.8
\tiny $\pm0.3$
&
\scriptsize
90.5
\tiny $\pm0.3$
&
\scriptsize
90.7
\tiny $\pm0.3$
&
\scriptsize
85.1
\tiny $\pm0.4$
&
\scriptsize
90.9
\tiny $\pm0.3$
&
\scriptsize
87.1
\tiny $\pm0.0$
&
\scriptsize
 - 
&
\scriptsize
\textbf{
91.8
}
\tiny $\pm0.3$
&
\scriptsize
-
\\
\midrule

{ \textsc{ isolet } } &

\scriptsize
0.0
\tiny $\pm0.0$
&
\scriptsize
0.0
\tiny $\pm0.0$
&
\scriptsize
96.3
\tiny $\pm0.7$
&
\scriptsize
95.6
\tiny $\pm1.1$
&
\scriptsize
93.0
\tiny $\pm0.9$
&
\scriptsize
91.4
\tiny $\pm1.0$
&
\scriptsize
\textbf{
98.0
}
\tiny $\pm0.7$
&
\scriptsize
97.4
\tiny $\pm0.6$
&
\scriptsize
97.6
\tiny $\pm0.6$
\\
\midrule

{ \textsc{ mammography } } &

\scriptsize
90.4
\tiny $\pm2.8$
&
\scriptsize
89.0
\tiny $\pm3.2$
&
\scriptsize
90.7
\tiny $\pm3.3$
&
\scriptsize
87.2
\tiny $\pm3.0$
&
\scriptsize
89.9
\tiny $\pm3.8$
&
\scriptsize
89.5
\tiny $\pm3.1$
&
\scriptsize
\textbf{
91.0
}
\tiny $\pm3.1$
&
\scriptsize
89.5
\tiny $\pm3.1$
&
\scriptsize
89.8
\tiny $\pm3.8$
\\
\midrule

{ \textsc{ protein homology } } &

\scriptsize
95.3
\tiny $\pm0.8$
&
\scriptsize
94.9
\tiny $\pm0.8$
&
\scriptsize
95.1
\tiny $\pm0.9$
&
\scriptsize
94.2
\tiny $\pm1.3$
&
\scriptsize
95.0
\tiny $\pm1.0$
&
\scriptsize
95.1
\tiny $\pm1.1$
&
\scriptsize
-
&
\scriptsize
\textbf{
95.7
}
\tiny $\pm0.9$
&
\scriptsize
 - 
\\
\midrule

{ \textsc{ webpages } } &

\scriptsize
72.0
\tiny $\pm0.0$
&
\scriptsize
73.1
\tiny $\pm2.0$
&
\scriptsize
93.0
\tiny $\pm1.8$
&
\scriptsize
\textbf{
93.1
}
\tiny $\pm1.7$
&
\scriptsize
86.7
\tiny $\pm0.0$
&
\scriptsize
84.4
\tiny $\pm1.6$
&
\scriptsize
-
&
\scriptsize
\textbf{
93.1
}
\tiny $\pm1.7$
&
\scriptsize
\textbf{
93.1
}
\tiny $\pm1.7$
\\
\bottomrule

 \end{tabular}

\end{sidewaystable*}
\begin{sidewaystable*}[ph!]
\centering
\caption{Big highly unbalanced datasets times}
\label{tab:bigt}
 \begin{tabular}{ccccccccccccc}
 
 \toprule
 &WELM$_{\textsc{sig}}$& EEM$_{\textsc{sig}}$ &WELM$_{\textsc{nsig}}$& EEM$_{\textsc{nsig}}$ &WELM$_{\textsc{rbf}}$   & EEM$_{\textsc{rbf}}$  &  LS-SVM$_{\textsc{rbf}}$ & EEKM$_{\textsc{rbf}}$ &  SVM$_{\textsc{rbf}}$\\
 \midrule
 
{ \textsc{ forest cover } } &

\scriptsize
110.7
s
&
\scriptsize
104.6
s
&
\scriptsize
144.9
s
&
\scriptsize
45.6
s
&
\scriptsize
111.3
s
&
\scriptsize
\textbf{
38.2
}
s
&
\scriptsize
$>$600 s
&
\scriptsize
107.4
s
&
\scriptsize
$>$600 s
\\
\midrule

{ \textsc{ isolet } } &

\scriptsize
9.7
s
&
\scriptsize
4.5 
s
&
\scriptsize
4.9
s
&
\scriptsize
3.0
s
&
\scriptsize
3.4
s
&
\scriptsize
\textbf{
2.1
}
s
&
\scriptsize
126.9
s
&
\scriptsize
3.2
s
&
\scriptsize
53.5
s
\\
\midrule

{ \textsc{ mammography } } &

\scriptsize
4.0
s
&
\scriptsize
\textbf{
2.2
}
s
&
\scriptsize
6.1
s
&
\scriptsize
3.0
s
&
\scriptsize
4.0
s
&
\scriptsize
2.2
s
&
\scriptsize
327.3
s
&
\scriptsize
3.3
s
&
\scriptsize
9.5
s
\\
\midrule

{ \textsc{ protein homology } } &

\scriptsize
27.6
s
&
\scriptsize
\textbf{
21.6
}
s
&
\scriptsize
86.3
s
&
\scriptsize
27.9
s
&
\scriptsize
62.5
s
&
\scriptsize
22.0
s
&
\scriptsize
$>$600 s
&
\scriptsize
30.7
s
&
\scriptsize
$>$600 s
\\
\midrule

{ \textsc{ webpages } } &

\scriptsize
16.0
s
&
\scriptsize
\textbf{
6.2
}
s
&
\scriptsize
14.5
s
&
\scriptsize
8.5
s
&
\scriptsize
7.1
s
&
\scriptsize
6.4
s
&
\scriptsize
$>$600 s
&
\scriptsize
9.0
s
&
\scriptsize
217.0
s
\\
\bottomrule

 \end{tabular}

\end{sidewaystable*}
%
%
%
%

\subsection{Entropy based hyperparameter optimization}

Now we proceed to entropy based evaluation. Given particular set of linear hypotheses $\mathcal{M}$ in $\Hspace $ we want to select optimal set of hyperparameters $\theta$ (such as number of hidden neurons or regularization parameter) which identify a particular model $\LinearOperator_\theta \in \mathcal{M} \subset \Hspace $. Instead of using expensive internal cross-validation (or other generalization error estimation technique like Err$^{0.632}$) we select such $\theta$ which maximizes our entropic measure. In particular we consider a simpified Cauchy-Schwarz Divergence based strategy where we select $\theta$ maximizing
$$
\Dcs(\nor(\LinearOperator_\theta^T \meanpoint^+, \var(\LinearOperator_\theta^T \Hlayer^+)),\nor(\LinearOperator_\theta^T \meanpoint^-, \var(\LinearOperator_\theta^T \Hlayer^-))),
$$
and
kernel density based entropic strategy~\cite{MELC} selecting $\theta$ maximizing
\begin{equation}
\Dcs( \de{ \LinearOperator_\theta^T \Hlayer^+ } , \de{ \LinearOperator_\theta^T \Hlayer^-}),
\label{dcscond}
\end{equation}
where $\de{ A } = \de{ A }_{\sigma(A)}$ is a Gaussian KDE using Silverman's rule of the window width~\cite{silverman1986density}
$$
\sigma(A) = \left ( \tfrac{4}{3|A|} \right )^{1/5}\std(A) \approx  \tfrac{1.06}{\sqrt[5]{|A|}}\std(A). 
$$

This way we can use whole given set for training and do not need to repeat the process, as $\Dcs$ is computed on the training set instead of the hold-out set. 

First, one can notice on Table~\ref{tab:enthyperdcsnor} and Table~\ref{tab:enthyperdcs} that such entropic criterion works well for EEM, EEKM and Support Vector Machines. On the other hand, it is not very well suited for ELM models. This confirms conclusions from our previous work on classification using $\Dcs$~\cite{MELC} where we claimed that SVMs are conceptually similar in terms of optimization objective, as well as widens it to the new class of models (EEMs). Second, Table~\ref{tab:enthyperdcsnor} shows that EEM and EEKM can truly select their hyperparameters using very simple technique requiring no model retrainings. Computation of (\ref{dcscond}) is linear in terms of training set and constant time if performed using precomputed projections of required objects (which are either way computed during EEM training). This make this 
very 
fast model even more robust.

\begin{sidewaystable*}[ph!]
\centering
\caption{\textsc{UCI} datasets \textsc{GMean} with parameters tuning based on selecting a model according to $\Dcs(\nor(\LinearOperator^T \meanpoint^+, \LinearOperator^T \Sigma^+ \LinearOperator),\nor(\LinearOperator^T \meanpoint^-, \LinearOperator^T \Sigma^- \LinearOperator))$ where $\LinearOperator$ is a linear operator found by a particular optimization procedure instead of internal cross validation}
\label{tab:enthyperdcsnor}
 \begin{tabular}{ccccccccccccc}
 
 \toprule
 &WELM$_{\textsc{sig}}$& EEM$_{\textsc{sig}}$ &WELM$_{\textsc{nsig}}$& EEM$_{\textsc{nsig}}$ &WELM$_{\textsc{rbf}}$   & EEM$_{\textsc{rbf}}$  &  LS-SVM$_{\textsc{rbf}}$ & EEKM$_{\textsc{rbf}}$ &  SVM$_{\textsc{rbf}}$\\
 \midrule
 
\textsc{ australian  } &

\scriptsize
51.2
\tiny $\pm 7.5$
&
\scriptsize
86.3
\tiny $\pm 4.8$
&
\scriptsize
50.3
\tiny $\pm 6.4$
&
\scriptsize
\textbf{
86.5
}
\tiny $\pm 3.2$
&
\scriptsize
50.3
\tiny $\pm 8.5$
&
\scriptsize
86.2
\tiny $\pm 5.3$
&
\scriptsize
58.5
\tiny $\pm 7.9$
&
\scriptsize
85.2
\tiny $\pm 5.6$
&
\scriptsize
85.7
\tiny $\pm 4.7$
\\
\midrule

\textsc{ breast-cancer  } &

\scriptsize
83.0
\tiny $\pm 4.3$
&
\scriptsize
97.0
\tiny $\pm 1.6$
&
\scriptsize
72.0
\tiny $\pm 6.6$
&
\scriptsize
97.1
\tiny $\pm 1.9$
&
\scriptsize
77.3
\tiny $\pm 5.3$
&
\scriptsize
97.3
\tiny $\pm 1.1$
&
\scriptsize
79.2
\tiny $\pm 7.7$
&
\scriptsize
96.9
\tiny $\pm 1.4$
&
\scriptsize
\textbf{
97.5
}
\tiny $\pm 1.2$
\\
\midrule

\textsc{ diabetes  } &

\scriptsize
52.3
\tiny $\pm 4.7$
&
\scriptsize
74.4
\tiny $\pm 4.0$
&
\scriptsize
51.7
\tiny $\pm 4.0$
&
\scriptsize
\textbf{
74.7
}
\tiny $\pm 5.2$
&
\scriptsize
52.1
\tiny $\pm 3.7$
&
\scriptsize
73.5
\tiny $\pm 5.9$
&
\scriptsize
60.1
\tiny $\pm 4.2$
&
\scriptsize
72.2
\tiny $\pm 5.4$
&
\scriptsize
73.2
\tiny $\pm 5.9$
\\
\midrule

\textsc{ german } &

\scriptsize
57.1
\tiny $\pm 4.0$
&
\scriptsize
69.3
\tiny $\pm 5.0$
&
\scriptsize
51.7
\tiny $\pm 3.0$
&
\scriptsize
\textbf{
72.4
}
\tiny $\pm 5.4$
&
\scriptsize
52.8
\tiny $\pm 6.3$
&
\scriptsize
70.9
\tiny $\pm 6.9$
&
\scriptsize
55.0
\tiny $\pm 4.3$
&
\scriptsize
67.8
\tiny $\pm 5.7$
&
\scriptsize
60.5
\tiny $\pm 4.5$
\\
\midrule

\textsc{ heart  } &

\scriptsize
68.6
\tiny $\pm 5.8$
&
\scriptsize
79.4
\tiny $\pm 6.9$
&
\scriptsize
65.6
\tiny $\pm 5.9$
&
\scriptsize
\textbf{
82.9
}
\tiny $\pm 7.4$
&
\scriptsize
60.3
\tiny $\pm 9.4$
&
\scriptsize
77.4
\tiny $\pm 7.2$
&
\scriptsize
66.2
\tiny $\pm 4.2$
&
\scriptsize
77.7
\tiny $\pm 7.0$
&
\scriptsize
76.5
\tiny $\pm 6.6$
\\
\midrule

\textsc{ ionosphere  } &

\scriptsize
62.7
\tiny $\pm 10.6$
&
\scriptsize
77.0
\tiny $\pm 12.8$
&
\scriptsize
68.5
\tiny $\pm 5.1$
&
\scriptsize
84.6
\tiny $\pm 9.1$
&
\scriptsize
69.5
\tiny $\pm 9.6$
&
\scriptsize
90.8
\tiny $\pm 5.2$
&
\scriptsize
72.8
\tiny $\pm 6.1$
&
\scriptsize
93.4
\tiny $\pm 4.2$
&
\scriptsize
\textbf{
94.7
}
\tiny $\pm 3.9$
\\
\midrule

\textsc{ liver-disorders  } &

\scriptsize
53.2
\tiny $\pm 7.0$
&
\scriptsize
68.5
\tiny $\pm 6.7$
&
\scriptsize
52.2
\tiny $\pm 11.8$
&
\scriptsize
62.1
\tiny $\pm 8.1$
&
\scriptsize
53.9
\tiny $\pm 8.0$
&
\scriptsize
\textbf{
71.4
}
\tiny $\pm 7.0$
&
\scriptsize
62.9
\tiny $\pm 7.8$
&
\scriptsize
69.6
\tiny $\pm 8.2$
&
\scriptsize
66.9
\tiny $\pm 8.0$
\\
\midrule

\textsc{ sonar  } &

\scriptsize
66.3
\tiny $\pm 6.1$
&
\scriptsize
66.1
\tiny $\pm 15.0$
&
\scriptsize
80.2
\tiny $\pm 7.4$
&
\scriptsize
76.9
\tiny $\pm 5.2$
&
\scriptsize
83.2
\tiny $\pm 6.9$
&
\scriptsize
82.8
\tiny $\pm 5.2$
&
\scriptsize
85.9
\tiny $\pm 4.9$
&
\scriptsize
\textbf{
87.7
}
\tiny $\pm 6.1$
&
\scriptsize
86.6
\tiny $\pm 3.3$
\\
\midrule

\textsc{ splice  } &

\scriptsize
51.8
\tiny $\pm 4.3$
&
\scriptsize
49.4
\tiny $\pm 5.5$
&
\scriptsize
64.9
\tiny $\pm 3.1$
&
\scriptsize
80.2
\tiny $\pm 2.6$
&
\scriptsize
60.8
\tiny $\pm 3.5$
&
\scriptsize
82.2
\tiny $\pm 3.5$
&
\scriptsize
\textbf{
89.7
}
\tiny $\pm 3.3$
&
\scriptsize
88.0
\tiny $\pm 4.0$
&
\scriptsize
89.5
\tiny $\pm 2.9$
\\
\midrule

 \end{tabular}
\end{sidewaystable*}

\begin{sidewaystable*}[ph!]
\centering
\caption{\textsc{UCI} datasets \textsc{GMean} with parameters tuning based on selecting a model according to $\Dcs(\de{ \LinearOperator^T \Hlayer^+ },\de{\LinearOperator^T \Hlayer^-})$ where $\LinearOperator$ is a linear operator found by a particular optimization procedure instead of internal cross validation}
\label{tab:enthyperdcs}
\begin{tabular}{ccccccccccccc}
 
 \toprule
 &WELM$_{\textsc{sig}}$& EEM$_{\textsc{sig}}$ &WELM$_{\textsc{nsig}}$& EEM$_{\textsc{nsig}}$ &WELM$_{\textsc{rbf}}$   & EEM$_{\textsc{rbf}}$  &  LS-SVM$_{\textsc{rbf}}$ & EEKM$_{\textsc{rbf}}$ &  SVM$_{\textsc{rbf}}$\\
 \midrule
 
\textsc{ australian  } &

\scriptsize
51.2
\tiny $\pm 7.5$
&
\scriptsize
86.3
\tiny $\pm 4.8$
&
\scriptsize
50.3
\tiny $\pm 6.4$
&
\scriptsize
\textbf{
86.5
}
\tiny $\pm 3.2$
&
\scriptsize
50.3
\tiny $\pm 8.5$
&
\scriptsize
86.2
\tiny $\pm 5.3$
&
\scriptsize
58.5
\tiny $\pm 7.9$
&
\scriptsize
85.2
\tiny $\pm 5.6$
&
\scriptsize
84.2
\tiny $\pm 4.1$
\\
\midrule

\textsc{ breast-cancer  } &

\scriptsize
83.0
\tiny $\pm 4.3$
&
\scriptsize
97.0
\tiny $\pm 1.6$
&
\scriptsize
72.0
\tiny $\pm 6.6$
&
\scriptsize
\textbf{
97.4
}
\tiny $\pm 1.2$
&
\scriptsize
77.3
\tiny $\pm 5.3$
&
\scriptsize
97.3
\tiny $\pm 1.1$
&
\scriptsize
79.3
\tiny $\pm 7.1$
&
\scriptsize
96.9
\tiny $\pm 1.4$
&
\scriptsize
96.3
\tiny $\pm 2.4$
\\
\midrule

\textsc{ diabetes  } &

\scriptsize
52.3
\tiny $\pm 4.7$
&
\scriptsize
74.4
\tiny $\pm 4.0$
&
\scriptsize
51.7
\tiny $\pm 4.0$
&
\scriptsize
\textbf{
74.7
}
\tiny $\pm 5.2$
&
\scriptsize
52.1
\tiny $\pm 3.7$
&
\scriptsize
73.5
\tiny $\pm 5.9$
&
\scriptsize
60.1
\tiny $\pm 4.2$
&
\scriptsize
72.2
\tiny $\pm 5.4$
&
\scriptsize
71.9
\tiny $\pm 5.4$
\\
\midrule

\textsc{ german } &

\scriptsize
57.1
\tiny $\pm 4.0$
&
\scriptsize
69.3
\tiny $\pm 5.0$
&
\scriptsize
51.7
\tiny $\pm 3.0$
&
\scriptsize
\textbf{
71.7
}
\tiny $\pm 5.9$
&
\scriptsize
52.8
\tiny $\pm 6.3$
&
\scriptsize
70.9
\tiny $\pm 6.9$
&
\scriptsize
54.4
\tiny $\pm 5.7$
&
\scriptsize
67.8
\tiny $\pm 5.7$
&
\scriptsize
59.5
\tiny $\pm 4.2$
\\
\midrule

\textsc{ heart  } &

\scriptsize
60.0
\tiny $\pm 9.2$
&
\scriptsize
79.4
\tiny $\pm 6.9$
&
\scriptsize
65.6
\tiny $\pm 5.9$
&
\scriptsize
\textbf{
82.9
}
\tiny $\pm 7.4$
&
\scriptsize
52.6
\tiny $\pm 9.0$
&
\scriptsize
77.4
\tiny $\pm 7.2$
&
\scriptsize
61.9
\tiny $\pm 5.8$
&
\scriptsize
77.7
\tiny $\pm 7.0$
&
\scriptsize
76.3
\tiny $\pm 7.7$
\\
\midrule

\textsc{ ionosphere  } &

\scriptsize
62.4
\tiny $\pm 8.1$
&
\scriptsize
77.0
\tiny $\pm 12.8$
&
\scriptsize
68.5
\tiny $\pm 5.1$
&
\scriptsize
84.6
\tiny $\pm 9.1$
&
\scriptsize
67.6
\tiny $\pm 9.8$
&
\scriptsize
90.8
\tiny $\pm 5.2$
&
\scriptsize
67.0
\tiny $\pm 10.7$
&
\scriptsize
\textbf{
93.4
}
\tiny $\pm 4.2$
&
\scriptsize
92.3
\tiny $\pm 4.6$
\\
\midrule

\textsc{ liver-disorders  } &

\scriptsize
50.9
\tiny $\pm 11.5$
&
\scriptsize
68.5
\tiny $\pm 6.7$
&
\scriptsize
50.4
\tiny $\pm 9.2$
&
\scriptsize
62.1
\tiny $\pm 8.1$
&
\scriptsize
53.9
\tiny $\pm 8.0$
&
\scriptsize
\textbf{
71.4
}
\tiny $\pm 7.0$
&
\scriptsize
62.9
\tiny $\pm 7.8$
&
\scriptsize
69.6
\tiny $\pm 8.2$
&
\scriptsize
66.9
\tiny $\pm 8.0$
\\
\midrule

\textsc{ sonar  } &

\scriptsize
66.3
\tiny $\pm 6.1$
&
\scriptsize
66.1
\tiny $\pm 15.0$
&
\scriptsize
80.2
\tiny $\pm 7.4$
&
\scriptsize
76.9
\tiny $\pm 5.2$
&
\scriptsize
62.9
\tiny $\pm 9.4$
&
\scriptsize
82.8
\tiny $\pm 5.2$
&
\scriptsize
83.6
\tiny $\pm 4.5$
&
\scriptsize
\textbf{
87.7
}
\tiny $\pm 6.1$
&
\scriptsize
86.6
\tiny $\pm 3.3$
\\
\midrule

\textsc{ splice  } &

\scriptsize
51.8
\tiny $\pm 4.3$
&
\scriptsize
33.1
\tiny $\pm 6.5$
&
\scriptsize
64.9
\tiny $\pm 3.1$
&
\scriptsize
80.2
\tiny $\pm 2.6$
&
\scriptsize
60.8
\tiny $\pm 3.5$
&
\scriptsize
82.2
\tiny $\pm 3.5$
&
\scriptsize
85.4
\tiny $\pm 4.1$
&
\scriptsize
88.0
\tiny $\pm 4.0$
&
\scriptsize
\textbf{
89.5
}
\tiny $\pm 2.9$
\\
\midrule

 \end{tabular}
\end{sidewaystable*}
\subsection{EEM stability}

It was previously reported~\cite{huang2006extreme} that ELMs have very stable results in the wide range of the number of hidden neurons. We performed analogous experiments with EEM on \textsc{UCI} datasets. We trained models for 100 increasing hidden layers sizes ($h=5,10,\ldots,500$) and plotted resulting \textsc{GMean} scores on Fig.~\ref{fig:stability}.

\begin{figure}[H]
\centering
 \includegraphics[width=10cm]{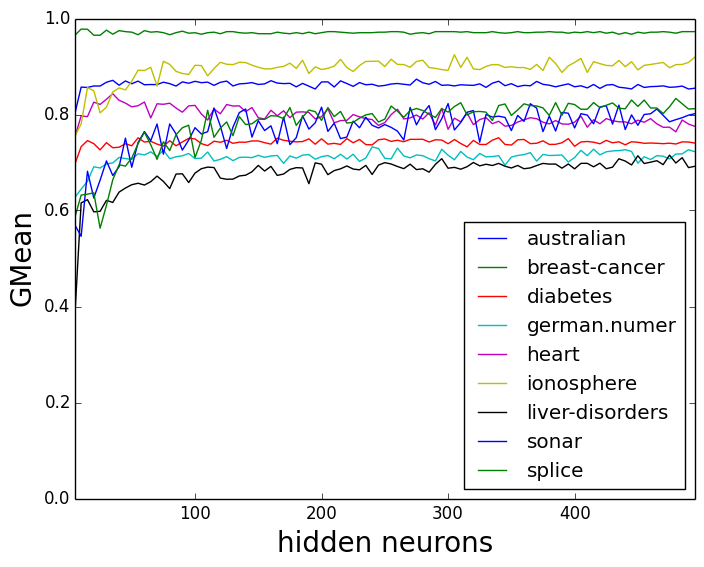}
 \caption{Plot of the EEM's (with RBF activation function) \textsc{GMean} scores from cross validation experiments for increasting sizes of hidden layer.}
 \label{fig:stability}
\end{figure}

One can notice that similarly to ELM proposed methods are very stable. Once machine gets enough neurons (around 100 in case of tested datasets) further increasing of the feature space dimension has minor effect on the generalization capabilities of the model. It is also important that some of these datasets (like sonar) do not even have 500 points, so there are more dimensions in the Hilbert space than we have points to build our covariance estimates, and even though we still do not observe any rapid overfitting.

\section{Conclusions}

In this paper we have presented Extreme Entropy Machines, models derived from the information theoretic measures and applied to the classification problems. Proposed methods are strongly related to the concepts of Extreme Learning Machines (in terms of general workflow, rapid training and randomization) as well as Support Vector Machines (in terms of margin maximization interpretation as well as LS-SVM duality). 

Main characteristics of EEMs are:
\begin{itemize}
 \item information theoretic background based on differential and Renyi's quadratic entropies,
 \item closed form solution of the optimization problem,
 \item generative training, leading to direct probability estimates,
 \item small number of metaparameters,
 \item good classification results,
 \item rapid training that scales well to hundreads of thousands of examples and beyond,
 \item theoretical and practical similarities to the large margin classifiers and Fischer Discriminant.
\end{itemize}

Performed evaluation showed that, similarly to ELM, proposed EEM is a very stable model in terms of the size of the hidden layer and achieves comparable classification results to the ones obtained by SVMs and ELMs. Furthermore we showed that our method scales better to truly big datasets (consisting of hundreads of thousands of examples) without sacrificing results quality.

During our considerations we pointed out some open problems and issues, which are worth investigation: 
\begin{itemize}
 \item Can one construct a closed-form entropy based classifier with different distribution families than Gaussians?
 \item Is there a theoretical justification of the stability of the extreme learning techniques?
 \item Is it possible to further increase achieved results by performing unsupervised entropy based optimization in the hidden layer?
\end{itemize}

\section*{Acknowledgment}

We would like to thank Daniel Wilczak from Institute of Computer Science and Computational Mathematics of Jagiellonian Univeristy for the access to the Fermi supercomputer, which made numerous experiments possible.

\bibliographystyle{plain}

\end{document}